\documentclass[twoside]{article}

\usepackage[accepted]{aistats2021}

\usepackage{amssymb}
\usepackage{amsmath}
\usepackage{amsthm}
\usepackage{xcolor}
\usepackage{algorithm}
\usepackage{algorithmic}
\usepackage{graphics}
\usepackage{comment}
\usepackage{bm}
\usepackage[round]{natbib}
\usepackage{dsfont}
\usepackage{ifmtarg}
\usepackage{xifthen}

\usepackage{hyperref}
\definecolor{Bleu}{RGB}{30,144,255}
\hypersetup{
colorlinks,
  citecolor=Bleu}

\usepackage{todonotes}

\newcommand*{\colorboxed}{}
\def\colorboxed#1#{%
  \colorboxedAux{#1}%
}
\newcommand*{\colorboxedAux}[3]{%
  \begingroup
    \colorlet{cb@saved}{.}%
    \color#1{#2}%
    \boxed{%
      \color{cb@saved}%
      #3%
    }%
  \endgroup
}

\newcommand{\R}{\mathbb{R}}
\newcommand{\N}{\mathbb{N}}

\newcommand{\bE}{\mathbb{E}}
\newcommand{\bP}{\mathbb{P}}
\newcommand{\bR}{\mathbb{R}}

\newcommand{\cA}{\mathcal{A}}
\newcommand{\cE}{\mathcal{E}}
\newcommand{\cF}{\mathcal{F}}
\newcommand{\cI}{\mathcal{I}}
\newcommand{\cS}{\mathcal{S}}
\newcommand{\cU}{\mathcal{U}}

\newcommand{\ARMS}{[K]}

\newcommand{\TOP}{\cS_{m}^{*,\varepsilon}}
\newcommand{\WORST}{(\TOP)^c}
\newcommand{\TOPM}{\cS_m^\star}
\newcommand{\WORSTM}{(\TOPM)^c}
\newcommand{\EMPTOP}[1]{\mbox{$\hat{S}^{#1}_m$}}

\newcommand{\EMPTOPTAU}{\mbox{$\hat{S}^{\tau_{\delta}}_m$}}
\newcommand{\EMPWORSTTAU}[1]{(\EMPTOPTAU)^c}

\newcommand{\GAP}[1]{\Delta_{#1}}
\newcommand{\GAPPAIRED}[2]{\Delta_{#1,#2}}
\newcommand{\EMPGAP}[3]{\hat{\Delta}_{#1, #2}(#3)}
\newcommand{\HA}[1]{\text{H}^{\varepsilon}(#1, \mu)}

\newcommand{\HATTHETA}[1]{\hat{\theta}^{\lambda}_{#1}}
\newcommand{\HATSIGMA}[1]{\hat{\Sigma}^{\lambda}_{#1}}
\newcommand{\HATB}[1]{\hat{V}^{\lambda}_{#1}}
\newcommand{\INVHATB}[1]{(\HATB{#1})^{-1}}

\newcommand{\EMPMU}[2]{\hat{\mu}_{#1}(#2)}

\newcommand{\INTERSEC}[1]{\underset{#1}{\bigcap}}

\newcommand{\EGIFA}{\cE^{GIFA}_m}
\newcommand{\tauUGapE}{\tau^{{UGapE}}}
\newcommand{\tauLUCB}{\tau^{{LUCB}}}

\DeclareMathOperator*{\argmax}{arg\,max}
\DeclareMathOperator*{\argmin}{arg\,min}

\newcommand{\mathscr}[1]{\mathcal{#1}}
\newcommand{\maxm}[1]{%
  \ifthenelse{\isempty{#1}}%
    {\overset{m}{\max}}
    {\underset{#1}{\overset{m}{\max}}\, }
}
\newcommand{\minm}[1]{%
  \ifthenelse{\isempty{#1}}%
    {\overset{m}{\min}}
    {\underset{#1}{\overset{m}{\min}}\, }
}
\newcommand{\argmaxm}[1]{%
  \ifthenelse{\isempty{#1}}%
    {\overset{m}{\argmax}}
    {\underset{#1}{\overset{m}{\argmax}}\, }
}
\newcommand{\argminm}[1]{%
  \ifthenelse{\isempty{#1}}%
    {\overset{m}{\argmin}}
    {\underset{#1}{\overset{m}{\argmin}}\, }
}
\newcommand{\maxmset}[1]{%
  \ifthenelse{\isempty{#1}}%
    {\overset{[m]}{\max}}
    {\underset{#1}{\overset{[m]}{\max}}\, }
}
\newcommand{\minmset}[1]{%
  \ifthenelse{\isempty{#1}}%
    {\overset{[m]}{\min}}
    {\underset{#1}{\overset{[m]}{\min}}\, }
}
\newcommand{\argmaxmset}[1]{%
  \ifthenelse{\isempty{#1}}%
    {\overset{[m]}{\argmax}}
    {\underset{#1}{\overset{[m]}{\argmax}}\, }
}
\newcommand{\argminmset}[1]{%
  \ifthenelse{\isempty{#1}}%
    {\overset{[m]}{\argmin}}
    {\underset{#1}{\overset{[m]}{\argmin}}\, }
}

\newcommand{\NA}[2]{N_{#1}(#2)}

\DeclareMathOperator*{\CUP}{\cup}
\DeclareMathOperator*{\CAP}{\cap}

\newcommand{\FREQUENTISTBETA}[1]{\sqrt{2\ln\left(\frac{1}{\delta}\right)+N\ln\left(1+\frac{(#1+1) L^2}{\lambda^2 N}\right)}+\frac{\sqrt{\lambda}}{\sigma}S}

\newtheorem{lemma}{Lemma}
\newtheorem{theorem}{Theorem}
\newtheorem{definition}{Definition}
\newtheorem{remark}{Remark}

\begin{document}
\runningtitle{Top-$m$ identification for linear bandits}
\runningauthor{R\'{e}da, Kaufmann, Delahaye-Duriez}


\twocolumn[\aistatstitle{Top-$m$ identification for linear bandits}
\aistatsauthor{ Cl\'{e}mence R\'{e}da$^{1,*}$ \And {E}milie Kaufmann$^2$ \And Andr\'{e}e Delahaye-Duriez$^{1,3,4}$ \\ \textcolor{white}{bla}}
\aistatsaddress{
$^1$ Universit\'{e} de Paris, Inserm UMR 1141 NeuroDiderot, F-75019, Paris, France, $^{*} \texttt{clemence.reda@inria.fr}$ \\
$^2$ Universit\'{e} Lille, CNRS, Inria, Centrale Lille, UMR 9189 CRIStAL, F-59000 Lille, France \\
$^3$ Universit\'{e} Sorbonne Paris Nord, UFR SMBH, F-93000, Bobigny, France \& $^4$ Assistance Publique des H\^{o}pitaux\\
de Paris, H\^{o}pital Jean Verdier, Service d'Histologie-Embryologie-Cytog\'{e}n\'{e}tique, F-93140, Bondy, France
}]
\vspace{-1cm}

\begin{abstract}
    Motivated by an application to drug repurposing, we propose the first algorithms to tackle the identification of the $m \geq 1$ arms with largest means in a linear bandit model, in the fixed-confidence setting. These algorithms belong to the generic family of Gap-Index Focused Algorithms (GIFA) that we introduce for Top-$m$ identification in linear bandits. We propose a unified analysis of these algorithms, which shows how the use of features might decrease the sample complexity. We further validate these algorithms empirically on simulated data and on a simple drug repurposing task.
\end{abstract}

\section{INTRODUCTION}

The multi-armed bandit setting, in which an agent sequentially gathers samples from $K$ unknown probability distributions called arms, is a powerful framework for sequential resource allocation tasks. While a large part of the literature focuses on the reinforcement learning problem in which the samples are viewed as reward that the agent seeks to maximize \citep{bubeckCB12survey}, \emph{pure-exploration} objectives have also received a lot of attention \citep{bubeck2009pure,degenne19multiple}. In this paper, we focus on \emph{Top-$m$ identification} in which the goal is to identify the $m < K$ arms with the largest expected rewards. While several Top-$m$ identification algorithms have been given, no algorithm has been specifically designed to tackle the  challenging \emph{linear bandit} setting~\citep{auer2002using}. This paper aims at filling this gap. In a linear bandit model, the mean $\mu_a$ of each arm $a$ is assumed to depend linearly on a known feature vector $x_a \in \bR^N$ associated to the arm: $\mu_a  = \theta^\top x_a$ for some vector $\theta \in \bR^N$. In contrast, the so-called classical bandit model does not make any assumption on the means $(\mu_a)_{a \in [K]}$\footnote{For $n\in \N^*$ we use the shorthand $[n] = \{1,\dots,n\}$.}. A Top-$m$ identification algorithm outputs a subset of size $m$ as a guess for the $m$ arms with largest means. Several objectives exist: in the fixed-budget setting, the goal is to minimize the probability that the guess is wrong after a pre-specified total number of samples from the arms \citep{bubeck2013multiple}. In this paper, we focus on the \emph{fixed-confidence setting}, in which the error probability should be guaranteed to be smaller than a given risk parameter $\delta \in (0,1)$ while minimizing the \emph{sample complexity}, that is, the total number of samples needed to output the guess. This choice is motivated by a real-life application to drug repurposing, in which we would like to control the failure rate in our predictions.

Drug repurposing is a field of research aimed at discovering new indications for drugs which are already approved for marketing, and may contribute to solve the problem of ever increasing research budget need for drug discovery~\citep{hwang2016failure}. For a given disease, we are interested in identifying a subset of drugs that may have a therapeutic interest. Providing a group of $5$ or $10$ drugs, rather than a single one, can ease the decision of further investigation, as many leads are provided. We believe sequential methods could be of interest, when considering a drug repurposing method called ``signature reversion''~\citep{musa2018review}. In this context, drugs recommended for repurposing are the ones minimizing the difference in gene activity between treated patients and healthy individuals. A possible way to measure this difference is to build a simulator that evaluates the genewise impact of a given drug on patients. However, this simulator may be stochastic and computationally expensive hence the need for sequential queries, that can be modeled as sampling of arms. As the arms (drugs) can be characterized by a real-valued feature vector which represents genewise activity change due to treatment, we resort to linear bandits to tackle this Top-$m$ identification problem.

\vspace{-0.3cm}

\paragraph{Related work} Two types of fixed-confidence algorithms have been proposed for Top-$m$ identification in a classical bandit: those based on adaptive sampling such as LUCB~\citep{kalyanakrishnan2012pac} or UGapE~\citep{gabillon2012best} or those based on uniform sampling and eliminations~\citep{kaufmann2013information,chen2017nearly}. For linear bandits, to the best of our knowledge, the only efficient algorithms have been proposed for the \emph{best arm identification} (BAI) problem, which corresponds to $m=1$. This setting, first investigated by \citet{soare14BAIlin}, recently received a lot of attention: an efficient adaptive sampling algorithm called LinGapE was proposed by \citet{xu2017fully} and subsequent work such as \citet{fiez2019sequential} sought to achieve the minimal sample complexity. In particular, the LinGame algorithm of \citet{degenne2020gamification} is proved to exactly achieve the problem-dependent sample complexity lower bound for linear BAI in a regime in which $\delta$ goes to zero. 

We note that, in principle, LinGame can be used for any pure exploration problem in a linear bandit, which includes Top-$m$ identification for $m>1$. However, this algorithm uses a game theoretic formalism which needs the computation of a best response for Nature in response to the player's selection; a computable expression of this strategy is not available to our knowledge for Top-$m$ ($m > 1$). Besides, computing the information-theoretic lower bound for Top-$m$ identification is also computationally hard. These remarks led us to investigate efficient adaptive sampling algorithms for general Top-$m$ identification $(m \geq 1)$ in linear bandit, which are still missing in the literature, instead of trying to propose asymptotically optimal algorithms, as done in linear BAI.

\vspace{-0.3cm}

\paragraph{Contributions} First, by carefully looking at known adaptive sampling bandits for classical Top-$m$, we propose a generic algorithm structure based on \emph{Gap Indices}, called GIFA, which encompasses existing adaptive algorithms for classical Top-$m$ identification and linear BAI. This structure allows a higher order and modular understanding of the learning process, and correctness properties can readily be inferred from a partially specified bandit algorithm. It allows us to define two interesting new algorithms, called $m$-LinGapE and LinGIFA. In Section~\ref{sec:gifa_theoretical}, we present a unified sample complexity analysis of 
a subclass of GIFA algorithms which comprises existing methods, which shows that the use of features can help decreasing the sample complexity in some cases. Finally, we show in Section~\ref{sec:experiments} that  $m$-LinGapE and LinGIFA perform better than their counterparts for classical bandits, both on artificially generated linear bandit instances and on a simple instance of our drug repurposing application, that is described in detail in Appendix~\ref{sec:dr_instance}.

\vspace{-0.4cm}

\paragraph{Notation} We let $\maxmset{}$, $\maxm{}$ (resp. $\minmset{}$, $\minm{}$) be the operator returning the $m$, $m^{th}$ greatest (resp. smallest) value(s), $\|x\|_{M} = \sqrt{x^\top M x}$ where $M$ is positive definite, and $\|x\| = \sqrt{x^\top x}$.

\section{SETTING}\label{sec:setting}

In this section, we introduce the \emph{probably approximately correct} (PAC) fixed-confidence linear Top-$m$ problem. From now on, we identify each arm by an integer in $\ARMS$ such that $\mu_1 \geq \mu_2 \geq \dots \geq \mu_m > \mu_{m+1} \geq \mu_{m+2} \geq \dots \geq \mu_K$. This ordering is of course unknown to the learner. In the PAC formalism, we are willing to relax the optimality constraint on arms using slack variable $\varepsilon > 0$ and we denote by $\TOP \triangleq \{a \in [K] : \mu_a \geq \mu_m - \varepsilon\}$ the set of best arms up to $\varepsilon$. 
Since $\mu_m > \mu_{m+1}$, we use the shorthand $\TOPM \triangleq \cS_m^{*,0} = [m]$ for the set of $m$ arms with largest means. The set of $K-m$ arms worst arms is its complementary $\WORSTM = [K] \backslash \TOPM$. A linear bandit model is parameterized by an unknown vector $\theta \in \R^N$, 
such that there are constants $L, S > 0$ which satisfy $\|\theta \| \leq S \in \bR^{*+}$ and such that the mean of arm $a$ is $\mu_a = \theta^\top x_a$, where the feature vector $x_a$ satisfies $\|x_a\| \leq L \in \bR^{*+}$. In each round $t \geq 1$, a learner selects an arm $a_t \in [K]$ and observes a sample  
\vspace{-0.3cm}
\[r_t = \theta^\top x_{a_t} + \eta_t,\]
where $\eta_t$ is a zero-mean sub-Gaussian noise with variance $\sigma^2$ --that is, $\bE[e^{\lambda \eta_t}] \leq \exp\left((\lambda^2\sigma^2)/{2}\right)$ for all $\lambda \in \R$ -- which is independent from past observations.

For $m\in [K]$, an algorithm for Top-$m$ identification consists of a \emph{sampling rule} $(a_t)_{t\in\N^*}$, where $a_t$ is $\cF_{t-1}$-measurable, and $\cF_{t} = \sigma(a_1,r_1,\dots,a_t,r_t)$ is the $\sigma$-algebra generated by the observations up to round $t$~; a \emph{stopping rule} $\tau$, which is a stopping time with respect to $\cF_t$ that indicates when the learning process is over; and a \emph{recommendation rule} $\EMPTOP{\tau}$, which is a $\cF_{\tau}$-measurable set of size $m$ that provides a guess for the $m$ arms with the largest means. Then, given a slack variable $\varepsilon \geq 0$ and the failure rate $\delta \in (0,1)$, an $(\varepsilon,m,\delta)$-PAC algorithm is such that 
\vspace{-0.3cm}
\[\bP\left(\EMPTOP{\tau} \subseteq  \cS_m^{*,\varepsilon}\right)  \geq 1 - \delta.\]
In the fixed confidence setting, the goal is to design a $(\varepsilon,m,\delta)$-PAC algorithm with a small \emph{sample complexity $\tau$}. In order to do so, the learner needs to estimate the unknown parameter $\theta \in \R^N$, which can be done with a (regularized) least-squares estimator. For any arm $a \in \ARMS$, we let $\NA{a}{t} \triangleq \sum_{s =1}^t \mathds{1}(a_s = a)$ be the number of times that arm $a$ is sampled up to time $t$, and define the $\lambda$-regularized design matrix and least-squares estimate: 
\vspace{-0.5cm}
\begin{eqnarray*}
\HATB{t} \!\!&\triangleq & \!\!\lambda I_N + \sum_{a=1}^{K} \NA{a}{t}x_ax_a^\top\\
\text{ and } \hat{\theta}_t^{\lambda}  &\triangleq& \left(\HATB{t}\right)^{-1}\left(\sum_{s=1}^{t} r_s x_{a_s}\right)\;. 
\end{eqnarray*}
We let $\HATSIGMA{t} \triangleq \sigma^{2}\left(\HATB{t}\right)^{-1}$, which can be interpreted as the posterior covariance in a Bayesian linear regression model in which the covariance of the prior is $(\sigma^2/\lambda) I_N$. 

The algorithms that we present in the next section crucially rely on estimating the gaps $\Delta_{i,j} \triangleq \mu_i - \mu_j$ between pairs of arms $(i,j)$, and building \emph{upper confidence bounds} (UCBs) on these quantities. For this purpose, we introduce the empirical mean of each arm $a$, $\EMPMU{a}{t} \triangleq (\HATTHETA{t})^\top x_a$ and define the empirical gap $\EMPGAP{i}{j}{t} \triangleq \EMPMU{i}{t}-\EMPMU{j}{t}$. A first option to build UCBs on $\Delta_{i,j}$ consists in building individual confidence intervals on the mean of each arm $a$, that are of the form $L_a(t) = \EMPMU{a}{t} - W_t(a)$ for the lower-confidence bound, and $U_a(t) = \EMPMU{a}{t} + W_t(a)$ for the upper confidence bound, where $W_t(a) \triangleq C_{t,\delta}\|x_a\|_{\HATSIGMA{t}}$ for some threshold function $C_{t,\delta}$ to be specified later. Clearly,
\vspace{-0.3cm}
\[B_{i,j}^{\text{ind}}(t) = U_i(t) - L_j(t) = \hat{\Delta}_{i,j}(t) + W_t(i)+W_t(j)\]
is an upper bound on $\Delta_{i,j}$ if $L_{j}(t) \leq \mu_j$ and $\mu_i \leq U_i(t)$. 
Yet, using the linear model, one can also directly build an UCB on the difference via 
\vspace{-0.3cm}
\[B_{i,j}^{\text{pair}}(t) = \hat{\Delta}_{i,j}(t) + C_{t, \delta}\|x_i - x_j\|_{\HATSIGMA{t}}.\]
Both constructions lead to symmetrical bounds\footnote{Note also that there exist algorithms with non symmetrical bounds, such the KL version of LUCB~\citep{kaufmann2013information}.}
of the form $B_{i,j}(t) = \hat{\Delta}_{i,j}(t)+W_t(i,j)$, where $W_t(i,j) = C_{t,\delta}\|x_i - x_j\|_{\HATSIGMA{t}}$ for paired UCBs and  $W_t(i,j) = C_{t, \delta}\left(\|x_i\|_{\HATSIGMA{t}} + \|x_j\|_{\HATSIGMA{t}}\right)$ for individual UCBs. In both cases, $W_t(i,j)=W_t(j,i) \geq 0$. 
The fact that these quantities are indeed upper confidence bounds (for some pairs $(i,j) \in \ARMS^2$) will be justified in Section~\ref{sec:algorithms}.

\begin{remark}
Observe that the linear bandit setting subsumes the classical bandit model, choosing arm features $(x_a)_{a \in \ARMS}$ to be vectors of the canonical basis of $\bR^K$ ($N=K$) and $\theta = \mu = [\mu_1, \mu_2, \dots, \mu_K]^\top$. Then, $\lambda = 0$ yields standard UCBs in this model: $\hat\mu_a(t)$ reduces to the empirical average of the rewards gathered from arm $a$ and $\|x_a\|_{\HATSIGMA{t}} = {\sigma}/{\sqrt{\NA{a}{t}}}$. 
\end{remark}

\section{GAP-INDEX FOCUSED ALGORITHMS (GIFA)}\label{sec:algorithms}

\subsection{Generic GIFA Algorithms}

\begin{table*}[t]
	\caption{Adaptive samplings for Top-$m$ (our proposals are in bold type~; except for LUCB and UGapE, all algorithms use indices $B=B^{\text{pair}}$ instead of individual indices $B^{\text{ind}}$).
    }\label{tab:algorithm_comparison}

    \vspace{0.1cm}

    \begin{tabular}{l|l|l|l|l}
        \textbf{Algorithm} & \texttt{compute\_Jt} & \texttt{compute\_bt}  & \texttt{selection\_rule} & \texttt{stopping\_rule}\\
        \hline
	    LUCB & $\colorboxed{white}{\argmax^{[m]}_{j \in \ARMS} \EMPMU{j}{t}}$ & $\colorboxed{white}{\argmax_{j \in J(t)} \max_{i \not\in J(t)} B_{i,j}(t)}$  & largest variance & $ B_{c_t,b_t}(t) \leq \varepsilon$\\
        UGapE & $\colorboxed{white}{\argmin^{[m]}_{j \in \ARMS} \max_{i \neq j}^{m} B_{i,j}(t)}$  & $\colorboxed{white}{\argmax_{j \in J(t)} \max_{i \not\in J(t)} B_{i,j}(t)}$ & largest variance & $\colorboxed{white}{\max_{j \in J(t)} \max_{i \neq j}^{m} B_{i,j}(t)}\leq \varepsilon$\\
        \hline
        LinGapE & $\colorboxed{white}{\argmax_{j \in \ARMS} \EMPMU{j}{t}}$ & $\colorboxed{white}{J(t)}$ ($|J(t)| = m = 1$) & greedy, optimized & $ B_{c_t,b_t}(t)\leq \varepsilon$ \\
        \hline
        \textbf{$\bm{m}$-LinGapE} & $\colorboxed{white}{\argmax^{[m]}_{j \in \ARMS} \EMPMU{j}{t}}$ & $\colorboxed{white}{\argmax_{j \in J(t)} \max_{i \not\in J(t)} B_{i,j}(t)}$ & largest variance,  & $ B_{c_t,b_t}(t)\leq \varepsilon$\\
	    \textbf{($m\geq1$)} & & & greedy, optimized \\
        \textbf{LinGIFA}  & $\colorboxed{white}{\argmin^{[m]}_{j \in \ARMS} \max_{i \neq j}^{m} B_{i,j}(t)}$ & $\colorboxed{white}{\argmax_{j \in J(t)} \max_{i \neq j}^{m} B_{i,j}(t)}$  & largest variance, & $\colorboxed{white}{\max_{j \in J(t)} \max_{i \neq j}^{m} B_{i,j}(t)}\leq \varepsilon$\\
        & & & greedy  &\\
    \end{tabular}
\end{table*}
Looking at the literature for Top-$m$ identification and BAI in (linear) bandits, existing adaptive sampling designs~\citep{kalyanakrishnan2012pac,gabillon2012best,xu2017fully} have many ingredients in common. We formalize in this section a generic structure that encompasses these algorithms, and under which we propose two new algorithms. We introduce the notion of  ``Gap-Index Focused'' Algorithms (GIFA), which associates to each \emph{pair of arm} $(i,j)$ an index $B_{i,j}(t)$ at time $t$ (called ``gap index''). 

The idea in GIFA is to estimate in each round $t$ a set of candidate $m$ best arms, denoted by $J(t)$, and to select the two most ambiguous arms: $b_t \in J(t)$, which can be viewed as a guess for the $m$-best arm, and a \emph{challenger} $c_t \not\in J(t)$. $c_t$ is defined as a potentially misassessed $m$-best arm, with largest possible gap to $b_t$: $c_t = \argmax_{c \in \ARMS} B_{c,b_t}(t)$. The idea of using two ambiguous arms goes back to LUCB~\citep{kalyanakrishnan2012pac} for Top-$m$ identification in classical bandits. Then the final arm $a_t$ selected by the algorithm should help discriminate between $b_t$ and $c_t$. A naive idea is to either draw $b_t$ or $c_t$, but alternative selection rules will be discussed later. At the end of the learning phase, at stopping time $\tau$, the final set $J(\tau)$ is recommended.

\begin{algorithm}[!h]
\begin{algorithmic}
\STATE \textbf{Input}: $K$ arms of means $\mu_1, \dots, \mu_K$; $\varepsilon \geq 0;$ 
\STATE $m \leq K; \delta \in (0,1)$.
\STATE \textbf{Output}: \EMPTOPTAU~estimated $m$ $\varepsilon$-optimal arms at final time $\tau_{\delta}$\\
\STATE $t \leftarrow 1$
\STATE \texttt{initialization}()
\WHILE{$\neg$ \texttt{stopping\_rule}($(B_{i,j}(t))_{i,j \in \ARMS}; \varepsilon, t$)}
\STATE \textit{// $J(t)$: estimated $m$ best arms at $t$}
\STATE $J(t) \leftarrow $ \texttt{compute\_Jt}($(B_{i,j}(t))_{i,j \in \ARMS}, (\EMPMU{i}{t})_{i \in \ARMS}$)
\STATE \textit{// $b_t$: estimated $m$-best arm at $t$}
\STATE $b_t \leftarrow $ \texttt{compute\_bt}($J(t), (B_{i,j}(t))_{i,j \in \ARMS}$)
\STATE \textit{// $c_t$: challenger to $b_t$ at $t$}
\STATE $c_t \leftarrow \argmax_{a \not\in J(t)} B_{a, b_t}(t)$ 
\STATE \textit{// selecting and pulling arms}
\STATE $a_t \gets$ \texttt{selection\_rule}($b_t, c_t ; \HATB{t-1}$)
\STATE $r_{t} \gets$ pulling($a_t$)
\STATE Update design matrix $\HATB{t}$, means $(\EMPMU{i}{t})_{i \in \ARMS}$
\STATE Update gap indices $(B_{i,j}(t+1))_{i,j \in \ARMS}$
\STATE $t \leftarrow t+1$
\ENDWHILE
\STATE \EMPTOPTAU $\leftarrow J(\tau_{\delta})$
\STATE \textbf{return} \EMPTOPTAU 
\end{algorithmic}
\caption{GIFA for ($\varepsilon,\delta$)-PAC Top-$m$}
\label{alg:adaptive_sampling_structure}
\end{algorithm}

The GIFA structure is presented in Algorithm~\ref{alg:adaptive_sampling_structure}. \texttt{initialization} is an optional phase where all arms are sampled once. We assume that ties are randomly broken. Degrees of freedom in designing the bandit algorithm lie in the choice of the rules \texttt{compute\_Jt}, \texttt{compute\_bt}, \texttt{selection\_rule} and \texttt{stopping\_rule}, some of which may also rely on the gap indices. For example, for the stopping rule, we restrict our attention to two stopping rules already proposed for the LUCB and UGapE algorithms, respectively: 
\begin{eqnarray*}
 \tauLUCB & \triangleq & \inf\left\{ t \in \N^* : B_{c_t, b_t}(t) \leq \varepsilon\right\}\;, \\
 \tauUGapE & \triangleq & \inf\left\{ t \in \N^* : \max_{j \in J(t)} \max_{i \neq j}^{m} B_{i,j}(t) \leq \varepsilon\right\}. 
\end{eqnarray*}

\subsection{Existing GIFA algorithms} 

We summarize in Table~\ref{tab:algorithm_comparison} existing and new algorithms that fit the GIFA framework. 

\paragraph{LUCB and UGapE} The only algorithms for Top-m identification in the fixed-confidence setting using adaptive sampling --namely, LUCB~\citep{kalyanakrishnan2012pac} and UGapE~\citep{gabillon2012best}-- were proposed for classical bandits. Both can be cast in the GIFA framework with the individual gap indices $B^{\text{ind}}_{i,j}(t) = U_i(t) - L_j(t)$ presented in Section~\ref{sec:setting}. Indeed, with this choice of indices, it can be shown that the ambiguous arms $b_t = \argmin_{j \in J(t)} L_j(t)$ and $c_t =\argmax_{i \not\in J(t)} U_i(t)$ used by both LUCB and UGapE can be rewritten as $\argmax_{j \in J(t)} \min_{i \not\in J(t)} B_{i,j}(t)$ and $\argmax_{i \not\in J(t)} B_{i,b_t}(t)$, respectively. This rewriting is crucial to propose extensions for the linear case. However, LUCB and UGapE differ by the set $J(t)$ they use and their stopping rule, as can be seen in Table~\ref{tab:algorithm_comparison}. LUCB uses the $m$ arms with largest empirical means for the set $J(t)$ while UGapE uses the $m$ arms with smallest values of $\maxm{i \neq j} \ B_{i,j}(t)$. This choice is motivated by \citet{gabillon2012best} by the fact that, with high probability, the quantity $\maxm{i \neq j} B_{i,j}(t)$ is an upper bound on $\mu_m - \mu_j$ . This is also true for more general gap indices, see Lemma~\ref{lemma:upper_bound_gap} in Appendix~\ref{app:technical}. 

This observation also justifies the UGapE stopping rule: when $\max_{j\in J(t)}\maxm{i \neq j} \ B_{i,j}(t) \  \leq \varepsilon$, with high probability $\mu_m - \mu_j \leq \varepsilon$ for all $j\in J(t)$, hence this set is likely to be included in $\cS_m^{*,\varepsilon}$. LUCB stops when $\max_{j\in J(t)} \max_{i \notin J(t)} B_{i,j}(t) \leq \varepsilon$. 
Interestingly, it follows from Lemma~\ref{lem:counting} below that, for fixed gap indices $(B_{i,j}(t))_{i,j \in \ARMS^2, t > 0}$, UGapE will always stop before LUCB. Still, we chose to investigate the two stopping rules, as the sample complexity analysis of UGapE 
uses the upper bound $\tauLUCB \geq \tauUGapE$. We provide correctness guarantees for both stopping rules in Section~\ref{subsec:correctness} and investigate their practical impact in Section~\ref{sec:experiments}.

\begin{lemma}\label{lem:counting}For all $t > 0$, for any subset $J$ of size $m$, for all $j\in J$, $\maxm{i\neq j} B_{i,j}(t) \leq \max_{i \notin J} B_{i,j}(t)$. \\
\textnormal{(proof in Appendix~\ref{app:technical})}
\end{lemma}

Regarding the selection rule, UGapE selects the least sampled arm among $b_t$ and $c_t$, which coincides with the largest variance rule that we propose for the general linear setting: 
\vspace{-0.3cm}
\begin{itemize}
 \item ``largest variance'': $\argmax_{a \in \{b_t,c_t\}}  W_t(a)$.
\end{itemize}
\vspace{-0.3cm}
In the original version of LUCB, both arms $b_t$ and $c_t$ are sampled at time $t$, but the analysis that we propose in this paper obtains similar guarantees for LUCB using the largest variance rule, so we only consider this selection rule in the remainder of the paper. 

\paragraph{LinGapE} A striking example of adaptive sampling in linear BAI (thus, only for $m=1$) is LinGapE~\citep{xu2017fully}. The first novelty in LinGapE compared to UGapE and LUCB is the use of paired gap indices $B^{\text{pair}}_{i,j}(t)$ exploiting the arm features. Paired indices may indeed increase performance. First, it follows from the triangular inequality for the Mahalanobis norm $\|\cdot \|_{\hat\Sigma_t^{\lambda}}$ that  $||x_i-x_j||_{\hat\Sigma_t^{\lambda}} \leq ||x_i||_{\hat\Sigma_t^{\lambda}}+||x_j||_{\hat\Sigma_t^{\lambda}}$ for all pairs of arms $(i,j)$~; therefore, if both types of gap indices use the same threshold $C_{t,\delta}$, paired indices are smaller: $B_{i,j}^{\text{pair}}(t) \leq B_{i,j}^{\text{ind}}(t)$. 
Moreover, Lemma~\ref{lemma:paired_versus_individual} below (proved in Appendix~\ref{app:technical}) implies that paired or individual indices using arm features always yield smaller bounds on the gaps than individual indices without arm features, that take the form 
\[\hat{\Delta}_{i,j}(t) + C_{t,\delta}\left[\frac{1}{\sqrt{N_i(t)}} + \frac{1}{\sqrt{N_j(t)}}\right]\;. \]

\begin{lemma}\label{lemma:paired_versus_individual}
$\forall t > 0, \forall a \in \ARMS, \forall y \in \bR^N$,\\
	\[||y||_{\INVHATB{t}} \leq ||y||/\left(\sqrt{\NA{a}{t}||x_a||^2+\lambda}\right).\]
\end{lemma}

The second novelty in LinGapE is the proposed selection rules. Let $X$ be the matrix which columns are the arm contexts, and $\|\cdot\|_1 : v \mapsto \sum_{i} |v_i|$. Instead of selecting either $b_t$ or $c_t$, \citet{xu2017fully} propose two selection rules to possibly sample another arm that would reduce the variance of the estimate $\theta^\top(x_{b_t} - x_{c_t})$:
\vspace{-0.3cm}
\begin{itemize}
	\item ``greedy'': $\argmin_{a \in \ARMS} ||x_{b_t}-x_{c_t}||_{(\HATB{t-1}+x_ax_a^\top )^{-1}}$,
	\item ``optimized'': if $\colorboxed{white}{w^*(b_t,c_t) = \!\!\!\!\!\argmin_{w \in \bR^K : x_{b_t}-x_{c_t} = wX^T} \!\!\|w\|_1}$,
	\begin{equation}\label{eq:optimized_strategy}
	\argmax_{a \in \ARMS : w^*_{a}(b_t,c_t) > 0} \NA{a}{t}\frac{||w^*(b_t,c_t)||_1}{|w^*_{a}(b_t,c_t)|}.
	\end{equation}
\end{itemize}
\vspace{-0.3cm}

As explained by \cite{xu2017fully}, these two rules are meant to bring the empirical proportions of selections close to an optimal design $(\lambda_1,\dots,\lambda_K)$ asymptotically minimizing $\|x_{b_t} - x_{c_t}\|_{\left(\sum_{a} \lambda_a x_ax_a^\top\right)^{-1}}$.

\subsection{$m$-LinGapE and LinGIFA} 

Building on the interesting features of existing algorithms, we now propose two new algorithms for Top-$m$ identification in linear bandits, called $m$-LinGapE and LinGIFA and described in Table~\ref{tab:algorithm_comparison}. 

$m$-LinGapE is an extension of LinGapE~\citep{xu2017fully} to linear Top-$m$, which coincides with the original algorithm for $m=1$. It may also be viewed as an extension of LUCB using paired indices. We investigate three possible selection rules for this algorithm: largest variance, greedy and optimized. LinGIFA is inspired by UGapE~\citep{gabillon2012best}, and also uses paired gap indices. Note that LinGIFA has a unique $\texttt{compute\_bt}$ rule $b_t = \argmax_{j \in J(t)} \maxm{i \neq j} \ B_{i,j}(t)$, which does not coincide with any of the previously mentioned rules. In Table~\ref{tab:algorithm_comparison}, we considered the stopping rules of the original algorithms they derive from. As LinGapE is one of the most performant algorithms for linear BAI (see for instance experiments in~\citet{degenne2020gamification}), we expect $m$-LinGapE to work well for Top-$m$. However, the strength of LinGIFA is that it is completely defined in terms of gap indices, and, as such, can easily be tuned for performance by deriving tighter bounds on the gaps. Moreover, LinGIFA relies on a stopping rule that is more aggressive than that used by $m$-LinGapE, as discussed above. 

We emphasize that provided that the regularizing constant $\lambda$ in the design matrix is positive, both $m-$LinGapE and LinGIFA can be run without an initialization phase (as done in the experiments for LinGIFA). This permits to avoid the initial sampling cost when the number of arms is large that was noticed by~\cite{fiez2019sequential}. 


\section{THEORETICAL GUARANTEES}\label{sec:gifa_theoretical}

In this section, we present an analysis of $m$-LinGapE and LinGIFA. The fact that these algorithms are $(\varepsilon,m,\delta)$-PAC is a consequence of generic correctness guarantees that can be obtained for GIFA algorithms (even not fully specified) and is presented in Section~\ref{subsec:correctness}. 
We further analyze in Section~\ref{subsec:SC} the sample complexity of LUCB-like algorithms, which comprise $m$-LinGapE.

\subsection{Correctness of GIFA instances} \label{subsec:correctness}

We justify that the two stopping rules introduced above lead to $(\varepsilon,m,\delta)$-PAC algorithms, provided a condition on the gap indices given in Definition~\ref{def:good_choice}. 

\begin{definition}[Good gap indices]\label{def:good_choice}
\textnormal{Let us denote \[\EGIFA \triangleq \bigcap_{t > 0} \bigcap_{j \in \WORST} \bigcap_{k \in \TOPM} \left(B_{k,j}(t) \geq \mu_k - \mu_j\right).\] A \textit{good choice} of gap indices $(B_{i,j}(t))_{i,j \in \ARMS, t > 0}$ satisfies $\bP(\EGIFA) \geq 1-\delta$.}
\end{definition}

First, we observe that on the event $\EGIFA$ introduced in Definition~\ref{def:good_choice} both stopping rules $\tauUGapE$ and $\tauLUCB$ output an $(\varepsilon, \delta)$ correct answer.

\begin{theorem}\label{th:correctness_GIFA}On the event $\EGIFA$: {\textbf{\emph{(i)}}} any GIFA algorithm using $b_t = \argmax_{j \in J(t)} \max_{i \not\in J(t)} B_{i,j}(t)$ satisfies $J(\tauLUCB) \subseteq \TOP$ and \textbf{\emph{(ii)}} any GIFA algorithm using $b_t \in J(t)$ satisfies $J(\tauUGapE) \subseteq \TOP$.
\end{theorem}

\begin{proof}
Let us assume by contradiction that there is an arm $b \in J(t) \CAP \WORST$ at stopping time $t$. We first prove that in both cases there exists $c \in \TOPM$ such that $B_{c,b}(t) \leq \varepsilon$ (*). Assuming that (*) does not hold, observe that $\TOPM \subseteq \{a \neq b, B_{a,b}(t) > \varepsilon\}$ \textbf{(a)} and then $\argmaxm{a \neq b} B_{a,b}(t) \in \{a \neq b,  B_{a, b}(t) > \varepsilon\}$ \textbf{(b)} since $b \not\in \TOPM$ and $|\TOPM| = m$. Depending on the definition of $b_t$ and on the stopping rule, we now split the proof into two parts:

\paragraph{\textbf{(i)}.} Using the definition of $\tauLUCB$, $c_t$ and $b_t$, it holds that $\forall b \in J(t), \forall c \not\in J(t), B_{c, b}(t) \leq \varepsilon$. Then, using \textbf{(a)}, $J(t)^c \subseteq \{a \neq b, B_{a,b}(t) \leq \varepsilon\}$, which means $\TOPM \CAP (J(t))^c = \emptyset$. Thus $\TOPM = J(t)$ (because $|\TOPM| = |J(t)| = m$) whereas $b \in J(t) \cap \WORSTM$, which is a contradiction. Hence (*) holds.
\paragraph{\textbf{(ii)}.} Using the definition of $\tauUGapE$, for any $b \in J(t)$, $\maxm{a \neq b}~B_{a,b}(t) \leq \epsilon$ holds. However, \textbf{(b)} means $\maxm{a \neq b}~B_{a,b}(t) > \epsilon$, which is a contradiction: (*) holds.

Then, at stopping time $t$, there exists $b \in J(t) \CAP \WORST$ and $ c \in \TOPM$ such that $B_{c, b}(t) \leq \varepsilon$. Using successively the definition of the event $\EGIFA$, and the fact that $c \in \TOPM$, this means that there exists $ b \in J(t)^c \CAP \WORSTM$ such that $\mu_b \geq \mu_m-\varepsilon$. Then $b \in \WORST \CAP \TOP$, which is absurd. Hence this proves that $J(t) \subseteq \TOP$ on the event $\EGIFA$ at stopping time $t$.
\end{proof}

It easily follows from Theorem~\ref{th:correctness_GIFA} that if $m$-LinGapE or LinGIFA are based on good gap indices in the sense of Definition~\ref{def:good_choice}, both algorithms are $(\varepsilon,m,\delta)$-PAC. We exhibit below a threshold for which the corresponding paired indices $B_{i,j}^{\text{pair}}(t)$ and individual indices $B_{i,j}^{\text{ind}}(t)$ (defined in Section~\ref{sec:setting}) are good gap indices. 

\begin{lemma}\label{lemma:frequentist_bound_correctness} For indices of the form $B_{i,j}(t) = \EMPGAP{i}{j}{t}+C_{\delta,t}W_t(i,j)$ with $W_t(i,j) = \|x_i-x_j\|_{\HATSIGMA{t}}$ (paired) or $W_t(i,j) = ||x_i||_{\HATSIGMA{t}}+||x_j||_{\HATSIGMA{t}}$ (individual)  with  
\begin{equation}
 C_{\delta,t}  =  \FREQUENTISTBETA{t}\label{def:PACthreshold},
\end{equation} 
where we recall that $\max_{a \in \ARMS} \|x_a\| \leq L$ and $\|\theta\|\leq S$, we have $\bP(\EGIFA) \geq 1 - \delta$.
\end{lemma}


\begin{proof} The proof follows from the fact that 
\[\left\{\forall t \in \N^* : \|\HATTHETA{t}-\theta\|_{(\HATSIGMA{t})^{-1}} \leq C_{\delta,t}\right\} \subseteq \EGIFA\]
together with Lemma 4.1 in \cite{kaufmann2014analyse} which yields $\bP\left(\forall t \in \N^* : \|\HATTHETA{t}-\theta\|_{(\HATSIGMA{t})^{-1}} \leq C_{\delta,t}\right) \geq 1-\delta$. For paired indices, the inclusion follows from the fact that 
\[\left|(\EMPMU{i}{t}-\EMPMU{j}{t}) - (\mu_i - \mu_j)\right| \leq \|\HATTHETA{t}-\theta\|_{(\HATSIGMA{t})^{-1}} \|x_i - x_j\|_{\HATSIGMA{t}},\]
where we use that $x^\top y \leq \|x\|_{\Sigma^{-1}} \|y\|_{\Sigma}$ for any $x,y \in \R^N$ and positive definite $\Sigma \in \R^{N\times N}$. For individual indices, we further need the triangular inequality $\|x_i - x_j\|_{\HATSIGMA{t}} \leq \|x_i\|_{\HATSIGMA{t}} + \|x_j\|_{\HATSIGMA{t}}$ to prove the inclusion.
\end{proof}

\subsection{Sample complexity results}\label{subsec:SC}

We derive below a high-probability upper bound on the sample complexity of a subclass of GIFA algorithms which comprise $m$-LinGapE
, combined with different selection rules. More precisely, we upper bound the 
sample complexity on
\[\cE \triangleq \bigcap_{t > 0} \bigcap_{i,j \in \ARMS} \Big(\mu_i-\mu_j \in [-B_{j,i}(t), B_{i,j}(t)]\Big),\]
an event which is trivially included in $\cE^{\text{GIFA}}$ and which holds therefore with probability larger than $1-\delta$ with the choice of threshold \eqref{def:PACthreshold}. To state our results, we define the true gap of an arm $k$ as $\GAP{k} \triangleq \mu_{k}-\mu_{m+1}$ if $k \in \TOPM$, $\mu_m-\mu_k$ otherwise ($\GAP{k} \geq 0$ for any $k \in \ARMS$).

\vspace{0.2cm}

\begin{theorem}\label{th:upper_bounds_linear_topm}
For $m$-LinGapE, 
on event $\cE$ on which algorithm $\cA$ is ($\varepsilon, m, \delta$)-PAC, stopping time $\tau_{\delta}$ satisfies $\tau_{\delta} \leq \inf \{u \in \bR^{*+}: u > 1+\HA{\cA}C_{\delta,u}^2 + \mathcal{O}(K)\}$, where:
\vspace{-0.3cm}
\paragraph{\textbf{(i).}} for $\cA$ $=$ $m$-LinGapE with the largest variance selection rule\footnote{or pulling both arms in $\{b_t,c_t\}$ at time $t$}~:
\vspace{-0.5cm}
\[\HA{\cA} \triangleq 4\sigma^2\colorboxed{white}{\sum_{a \in \ARMS}} \max \left(\varepsilon,\frac{\varepsilon+\GAP{a}}{3} \right)^{-2},\]
\vspace{-1cm}
\paragraph{\textbf{(ii).}} for $\cA$ $=$ $m$-LinGapE with the optimized selection rule in~\citet{xu2017fully}~:
\vspace{-0.5cm}
\[\HA{\cA} \triangleq \sigma^2\colorboxed{white}{\sum_{a \in \ARMS}}  \max_{i,j \in \ARMS} \frac{|w^*_a(i,j)|}{\max \left(\varepsilon, \frac{\varepsilon+\GAP{i}}{3}, \frac{\varepsilon+\GAP{j}}{3} \right)^2},\] 
where $w^*(i,j)$ satisfies Equation~\ref{eq:optimized_strategy} applied to the arm pair $(i,j)$.
\end{theorem}

\vspace{-0.2cm}

Upper bounds for all analyzed algorithms and classical counterparts are shown in Table~\ref{tab:sample_complexity_algorithms}.

\begin{table*}[t]
    \centering
    \caption{Sample complexity results for linear and classical Top-$m$ algorithms. $w^*$ is defined as in Equation~\eqref{eq:optimized_strategy}. The additive term in $K$ is due to the initialization phase. Our proposal is in bold type.}
    
    \vspace{0.2cm}
    
    \label{tab:sample_complexity_algorithms}
    \begin{tabular}{l|l|l}
        \textbf{Algorithm} & \textbf{Complexity constant $\HA{\cdot}$} & \textbf{Upper bound on $\tau_{\delta}$}\\
        \hline
	    LUCB  & $\colorboxed{white}{2\sum_{a \in \ARMS} \max \left(\frac{\varepsilon}{2},{\GAP{a}} \right)^{-2}}$ & $\inf_{u > 0} \left\{u > 1+\HA{\text{LUCB}}C_{\delta,u}^2 + K\right\}$\\
	    UGapE & $\colorboxed{white}{2\sum_{a \in \ARMS} \max \left(\varepsilon,\frac{\varepsilon+\GAP{a}}{2} \right)^{-2}}$ & $\inf_{u>0} \left\{u > 1+\HA{\text{UGapE}}C_{\delta,u}^2 + K\right\}$\\
         \hline
	    \textbf{$m$-LinGapE (1)} & $\colorboxed{white}{4\sigma^2\sum_{a \in \ARMS} \max \left(\varepsilon,\frac{\varepsilon+\GAP{a}}{3} \right)^{-2}}$ & $\inf_{u>0} \left\{ u > 1+\HA{\text{$m$-LinGapE(1)}}C_{\delta,u}^2\right\}$\\
	    \textbf{(largest variance)} & & \\
	    \textbf{$m$-LinGapE (2)} & $\colorboxed{white}{\sigma^2\sum_{a \in \ARMS} \max_{i,j \in \ARMS} \frac{|w^*_a(i,j)|}{\max \left(\varepsilon, \frac{\varepsilon+\GAP{i}}{3}, \frac{\varepsilon+\GAP{j}}{3} \right)^2}}$ & $\inf_{u>0} \left\{u > 1+\HA{\text{$m$-LinGapE(2)}}C_{\delta,u}^2 \right\}$\\
	    \textbf{(optimized)} & & \\
         & & \\
    \end{tabular}
    
    \vspace{-0.2cm}
\end{table*}

\subsection{Sketch of proof}

The proof of Theorem~\ref{th:upper_bounds_linear_topm} generalizes and extends the proofs for classical Top-$m$ and linear BAI
~\citep{xu2017fully}, with paired or individual gap indices. We sketch below the sample complexity analysis of a specific subclass of GIFA algorithms, 
postponing the proof of some auxiliary lemmas to Appendix~\ref{sec:upper_bounds}. 
We 
focus on GIFA algorithms that use 
\[J(t) \triangleq \argmaxmset{i \in \ARMS}~\EMPMU{i}{t} \ \text{ and } \ b_t \triangleq \argmax_{j \in J(t)} \max_{i \notin J(t)} B_{i,j}(t)\;. \]
This includes LinGapE, $m$-LinGapE, but also LUCB. The key ingredient in the proof is the following lemma, which holds for any gap indices of the form $B_{i,j}(t) \triangleq \EMPMU{i}{t}-\EMPMU{j}{t} + W_t(i,j)$ for $i,j \in \ARMS^2$.

\begin{lemma}\label{lemma:m-LinGapE_bound}
On the event $\cE$, for all $t > 0$, 
\[B_{c_t, b_t}(t) \leq \min(-({\GAP{{b_t}}} \lor {\GAP{{c_t}}})+2W_t(b_t,c_t), 0)+W_t(b_t,c_t).\]
\end{lemma}

This result is a counterpart of Lemma $4$ in~\citet{xu2017fully}, but does not require $|J(t)|=1$ at a given time $t > 0$, notably by noticing that, by definition of $b_t$ and $c_t$, $B_{c_t,b_t}(t) = \max_{j \in J(t)} \max_{i \not\in J(t)} B_{i,j}(t)$. In order to get the upper bound in Theorem~\ref{th:upper_bounds_linear_topm} for $m$-LinGapE using the optimized rule, one can straightforwardly apply Lemma $1$ in~\citet{xu2017fully} to the inequality stemming from Lemma~\ref{lemma:m-LinGapE_bound}. For the selection rules which either select both $b_t$ and $c_t$, or the arm in $\{b_t,c_t\}$ with the largest variance term, by combining Lemma~\ref{lemma:m-LinGapE_bound} with the definition of the stopping rule $\tauLUCB$, we obtain the following upper bound on $\NA{a_t}{t}$, where $a_t$ is (one of) the pulled arm(s) at time $t < \tau_{\delta}$. 

\begin{lemma}\label{lemma:get_upper_bound_na_variance}
$\forall t > 0, \tau_{\delta} > t, \NA{a_t}{t} \leq T^{*}(a_t, \delta, t)$, where $a_t$ is a pulled arm at time $t$, and 

\vspace{-0.5cm}

\[T^{*}(a_t, \delta, t) = 4\sigma^2 C_{\delta,t}^2 \max \left( \varepsilon, \frac{\varepsilon+\GAP{a_{t}}}{3}\right)^{-2}.\]
\end{lemma}

Finally, the upper bound in Theorem~\ref{th:upper_bounds_linear_topm} is a consequence of the following result.

\begin{lemma}\label{lemma:get_upper_bound_sample}
Let $T^{*} : \ARMS \times (0,1) \times \mathbb{N}^{*} \rightarrow \bR^{*+}$ be a function that is nondecreasing in $t$, and $\cI_t$ the set of pulled arms at time $t$. Let $\cE$ be an event such that for all $t < \tau_{\delta}, \delta \in (0,1)$, $\exists a_t \in \cI_t, \NA{a_t}{t}\leq T^{*}(a_t, \delta, t) $. Then it holds on the event $\cE$ that $\tau_\delta \leq T(\mu,\delta)$ where \[T(\mu, \delta) \triangleq \inf \left\{u \in \bR^{*+}: u > 1+\sum_{a=1}^{K} T^{*}(a, \delta, u) \right\}.\]
\end{lemma}

\subsection{Discussion}\label{sec:discussion_exp}

If $\HA{\cA}$ is the complexity constant associated with bandit algorithm $\cA$ and bandit instance $\mu \in \bR^{K}$, then, by looking at Table~\ref{tab:sample_complexity_algorithms}, we can first notice that 
\[\HA{\text{LUCB}} \geq \HA{\text{$m$-LinGapE(1)}} \geq \HA{\text{UGapE}}.\]

Note that the quantity $w^*_a(i,j)$ featured in the complexity quantity associated to $m$-LinGapE with the optimized selection rule ($m$-LinGapE(2)) is not \emph{a priori} bounded by a constant -- see Equation~\eqref{eq:optimized_strategy}. Authors in~\citep{xu2017fully} have shown that in specific instances, $\HA{\text{$m$-LinGapE(2)}} \leq \frac{9}{8}\HA{\text{UGapE}}$. To further compare these two complexity quantities, we designed the following experiment to estimate how many times $\HA{\text{$m$-LinGapE(2)}} \leq \HA{\text{UGapE}}$: $K \times N$ values are randomly sampled from a Gaussian distribution $N(0, D)$. From these we build matrix $X' \in \bR^{N \times K}$, which is then column-normalized (norm $\|\cdot\|$) to yield feature matrix X. We use $\theta = e_1$, where $e_1 = [1, 0, \dots, 0]^\top  \in \bR^K$ and $\|\theta\|=1$. Then we compute both constants $\HA{\text{$m$-LinGapE(2)}}$ and $\HA{\text{UGapE}}$ where $\mu = X\theta$ and $m=\left[ \frac{K}{3} \right]$, check whether $\HA{\text{$m$-LinGapE(2)}} \leq \HA{\text{UGapE}}$. We perform this experiment $1,000$ times and report in Table~\ref{tab:comparison_complexity} (in Appendix~\ref{sec:complexity_constants}) the fraction of times this condition holds, for multiple values of $K$ (number of arms), $N$ (dimension) and $D$ (variance of the Gaussian distribution). We observe that, in these artificially generated instances, the condition $\HA{\text{$m$-LinGapE(2)}} \leq \HA{\text{UGapE}}$ is seldom verified. 

Hence, our theoretical analysis suggests that, in these cases --which usually turn out to be really hard linear instances-- LinGIFA, which structure is similar to that of UGapE, may perform better than $m$-LinGapE. However, as we will see below, the practical story is different and $m$-LinGapE with the optimized selection rule can be as efficient, and often better than LinGIFA using the ``largest variance'' selection rule.

\section{EXPERIMENTAL STUDY}\label{sec:experiments}

In this section, we report the results of experiments on simulated data and on a simple drug repurposing instance. In all experiments, we use $\sigma = 1/2$, $\varepsilon = 0$, $\delta=5\%$. For all artificially generated experimental settings, we use Gaussian arms, that is, at time $t$, $r_{t} \thicksim N(\mu_{a_t}, \sigma^2)$. For the tuning of our algorithm we set the regularization parameter to $\lambda = \sigma/20$, following \citet{hoffman2014correlation}. In an effort to respect the trade-off between identification performance and speed, we set $C_{\delta,t} = \sqrt{2\ln \left( (\ln(t)+1)/\delta \right)}$ (``heuristic'' value) if the empirical error remained below $\delta$ instead of the theoretically valid threshold \eqref{def:PACthreshold} that guarantees all algorithms to be $(\varepsilon, m, \delta)-$PAC. Otherwise, we use \eqref{def:PACthreshold}, and for classical bandits, we use the threshold used in~\citet{kalyanakrishnan2012pac}. We consider the following experimental settings.

\paragraph{Classic instances} We consider $K$ arms such that, for $a\in[m]$, $x_a = e_1+e_a$, $x_{m+1} = \cos(\omega)e_1 + \sin(\omega)e_{m+1}$, and for any $a > m+1$, $x_a = e_{a-1}$, and $\theta = e_1$, where $(e_i)_{i \in [K-1]}$ are the canonical basis of $\bR^{K-1}$. This construction extends the usual ``hard'' instance for BAI in linear bandits ($m=1$), originally proposed by \citet{soare14BAIlin} where one considers three arms, of respective feature vectors $x_1 = [1, 0]^\top , x_2 = [0, 1]^\top $ and $x_3 = [\cos(\omega), \sin(\omega)]^\top $, and $\theta = e_1$, where $\omega \in (0, \frac{\pi}{2})$. Arm $1$ is the best arm, and as $\omega$ decreases, it becomes harder to discriminate between arms $1$ and $3$, and as such, pulling suboptimal arm $2$ might be useful.
\vspace{-0.3cm}
\paragraph{Drug repurposing instance} Our drug-scoring function relies on the simulation of the treatment effect on gene activity via a Boolean network~\citep{kauffman1969metabolic}. The details are laid out in Appendix~\ref{sec:dr_instance}. When a drug is selected, a reward is generated by applying the simulator to the gene activity profile of a patient chosen at random, and computing the cosine score between the post-treatment gene activity profile and the healthy gene activity profile. The higher this score is, the most similar the final treated patient and healthy samples are. We consider as drug features the treated gene activity profiles, which belong to $\R^{71}$.
\vspace{-0.3cm}
\paragraph{Results} Each figure reports the empirical distribution of the sample complexity (estimated over $500$ runs) for different algorithms. In all the experiments, the empirical error is always below $5\%$, and is reported in Table~\ref{tab:error_frequencies}. 


Figure~\ref{fig:boxplots_a}(a) reports the results for a classic problem with $m=2, K=4$ and $N=3$. We compare two algorithms that are not using the arm features, LUCB and UGapE, with GIFA algorithms based on contextual indices. We investigate the use of individual and paired gap indices, different selection rules, and the two stopping rules $\tauLUCB$ and $\tauUGapE$ in both $m$-LinGapE and LinGIFA. We observe that there is almost no difference in the use of either stopping rules, but using paired indices leads to noticeably better sample complexity than individual ones. As noticed in~\citet{xu2017fully} for $m=1$, $m$-LinGapE with the optimized and the greedy sampling rules have similar performance. In addition, using the greedy or optimized rules leads to slightly better performance compared to the largest variance rule. More importantly, we observe that linear bandit algorithms largely outperform their classical counterparts, even on a rather easy instance ($\omega=\frac{\pi}{6}$, hence $\mu_{m}-\mu_{m+1} \approx 0.13$). 

 Figure~\ref{fig:boxplots_a}(b) shows the empirical sample complexity on a hard Top-$1$ classic problem ($\omega =0.1$, $\mu_1-\mu_{2} \approx 0.005$) which allows to compare our algorithms to the state-of-the-art LinGame algorithm. $m-$LinGapE and LinGIFA have a better performance than LinGame, which is encouraging, even if our algorithms do not have any optimality guarantees in theory. 
 
 Figure~\ref{fig:boxplots_a}(c) displays the result for a drug repurposing instance for epilepsy, including $K=10$ arms (drugs) with $5$ anti-epileptic and $5$ pro-convulsant drugs, hence the choice $m=5$ (a close-up is shown in Appendix~\ref{sec:dr_instance} without LUCB). Although the linear dependency between rewards and features may be violated in this real-world example, taking into account features still helps in considerably reducing the sample complexity (approximately by a factor $\frac{1}{2}$). 

\begin{table}
	\caption{Error frequencies rounded to $5$ decimal places, for each Top-$m$ and BAI algorithm (averaged across $500$ runs). Proposed algorithms' names are in bold type. Each column corresponds to one figure.}
	\label{tab:error_frequencies}
	\begin{tabular}{l|c|c|c}
		\textbf{Algorithm} & \textbf{(1)a} & \textbf{(1)b} & \textbf{(1)c}\\
		\hline
		\textbf{$m$-LinGapE (greedy)} & $0.0$ & $0.044$ & $0.0$ \\
		\textbf{$m$-LinGapE (optimized)} & $0.0$ & - & $0.0$\\
		\textbf{LinGIFA (largest var.)} & $0.0$ & - & $0.0$\\
		\textbf{LinGIFA} & & & \\
		\textbf{($\tauLUCB$, largest var.)} & $0.0$ & - & -\\
		\textbf{LinGIFA ($\tauLUCB$, greedy)} & $0.0$ & - & -\\
		\textbf{LinGIFA (greedy)} & $0.0$ & $0.0$ & $0.0$\\
		\textbf{LinGIFA} & & & \\
		\textbf{(individual indices)} & $0.0$ & - & -\\
		LUCB (largest var.) & $0.0$ & - & $0.0$\\
		LUCB (sampling both arms) & - & - & $0.0$\\
		UGapE & $0.0$ & - & -\\
		LinGame & - & $0.0$ & -\\
	\end{tabular}
\vspace{-0.5cm}
\end{table}

\begin{figure}[h]
\includegraphics[scale=0.16]{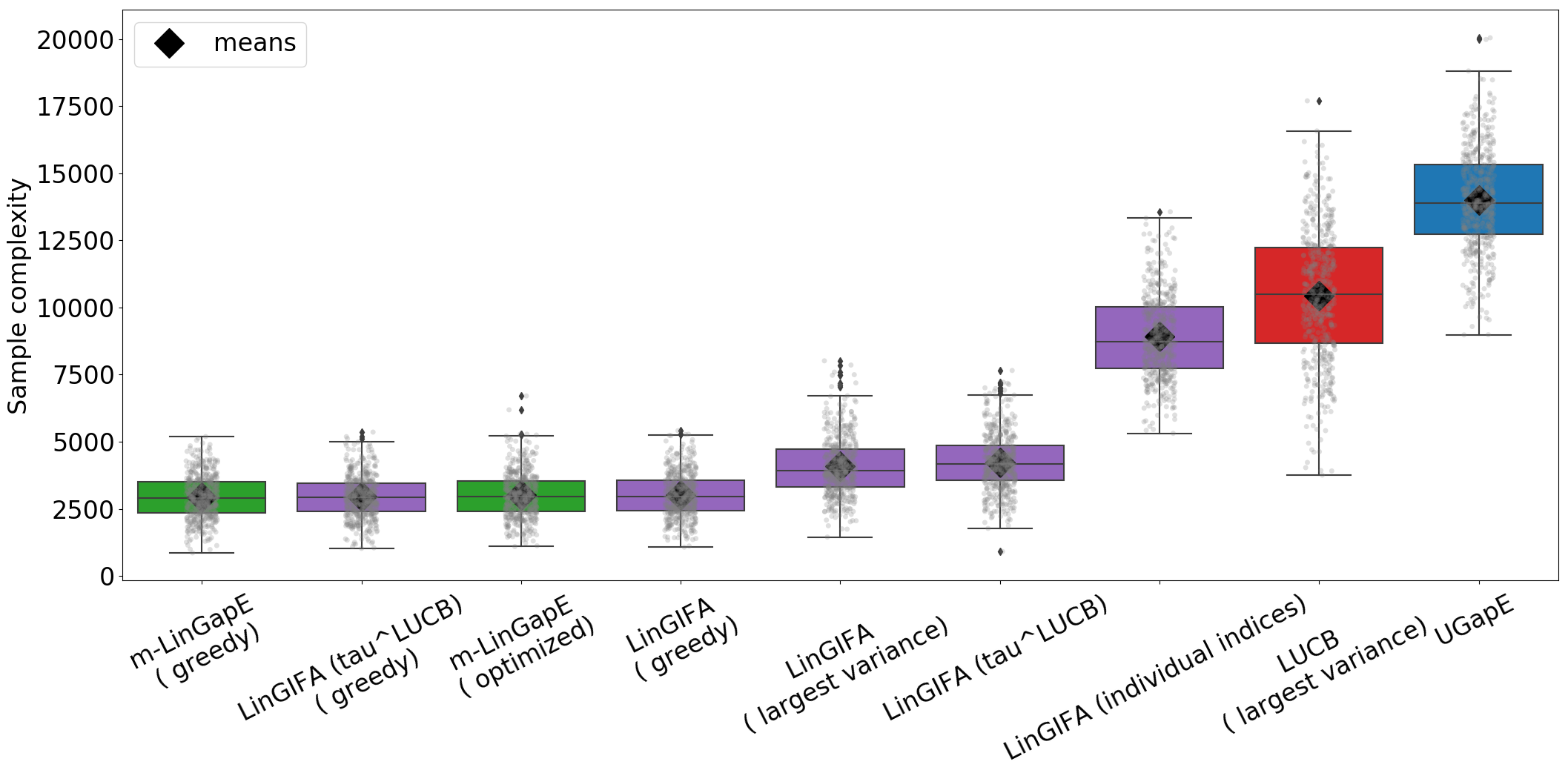}
\includegraphics[scale=0.16]{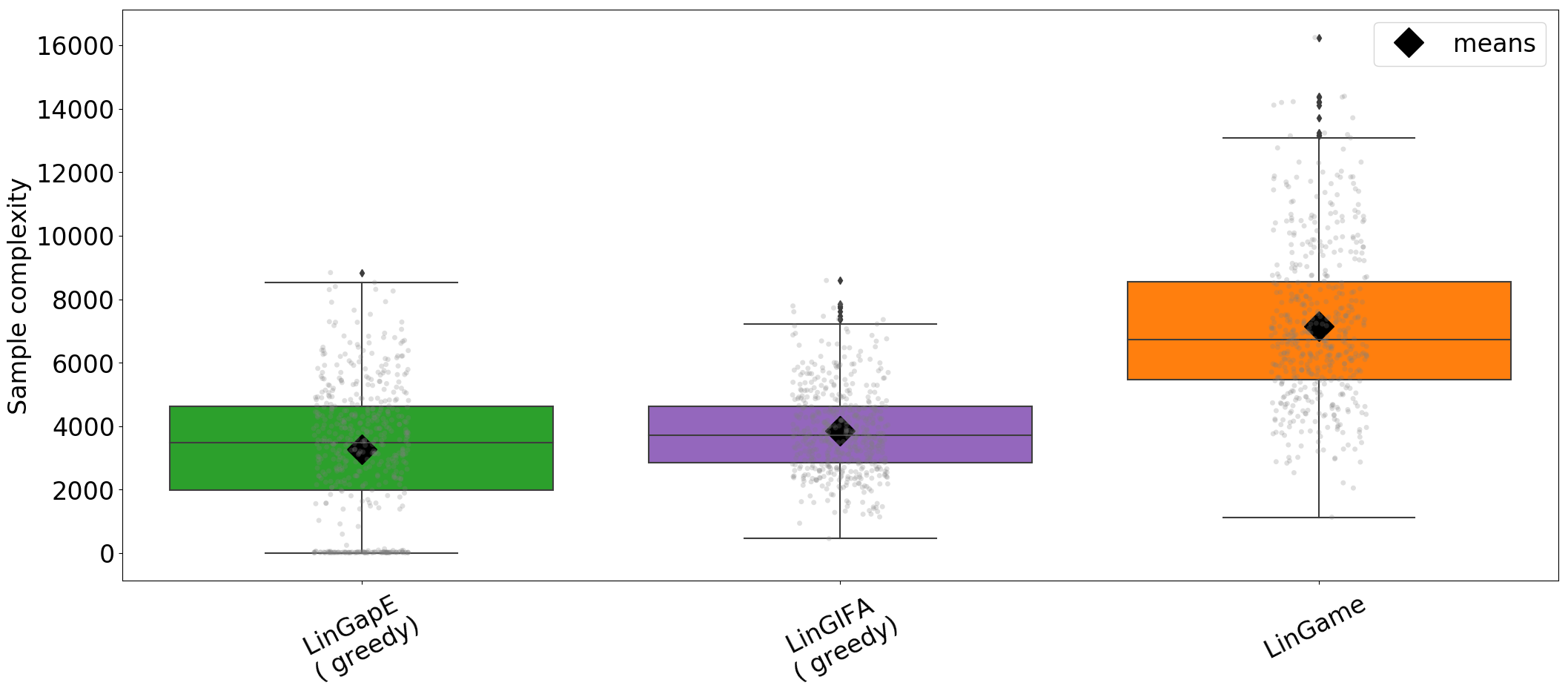}
\includegraphics[scale=0.16]{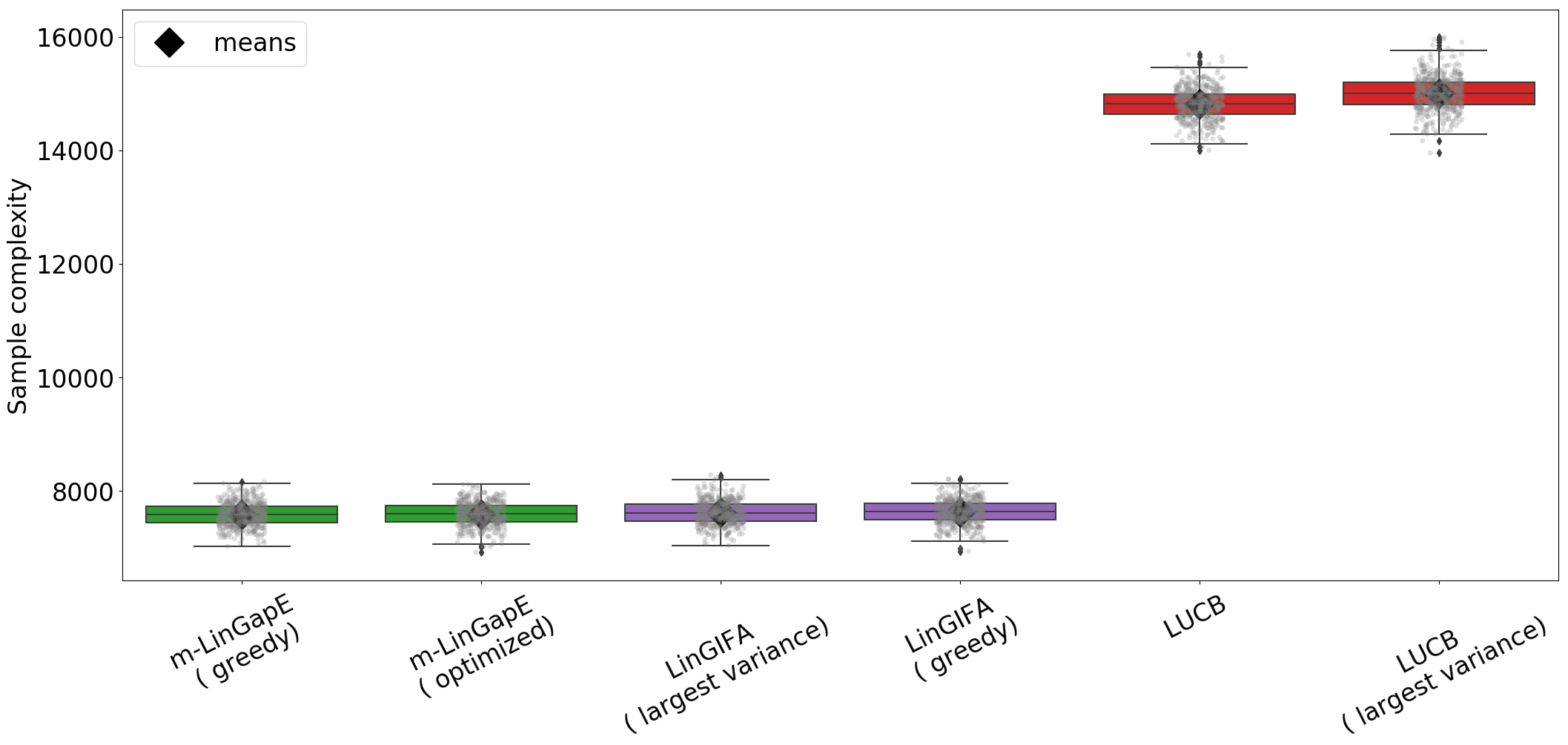}
\caption{From top to bottom: classic instances (a) $K=4$, $\omega=\frac{\pi}{6}$, $m=2$~; (b) $K=3$, $\omega=0.1$, $m=1$~; (c) drug repurposing instance $K=10$, $m=5$. Lines are quantiles and jittered individual outcomes are plotted in grey.}
\label{fig:boxplots_a}
\vspace{-0.5cm}
\end{figure}

\section{DISCUSSION}\label{sec:discussion}

To our knowledge, we have provided the first unified framework and fully adaptive algorithms for linear Top-$m$. Our theoretical analysis shows our algorithms do not perform any worse than their classical counterparts. Code is publicly available at \texttt{https://github.com/clreda/linear-top-m}. However, in real life and in our drug repurposing instance, the linear dependency between features and rewards does not hold. A future direction of our work would be dealing with such model misspecification. Another perspective would be the analysis of the greedy sampling rule. Indeed, this sampling rule leads to more efficient algorithms in our linear experiments.

\section*{ACKNOWLEDGEMENTS}
	The authors thank the anonymous reviewers for their feedback, Xuedong Shang for his valuable input for the implementation of LinGame, and Baptiste Porte for his help on building the drug repurposing instance applied to epilepsy. This work was supported by the Institut National pour la Sant\'{e} et la Recherche M\'{e}dicale (INSERM, France), the Centre National de la Recherche Scientifique (CNRS, France), the Universit\'{e} Sorbonne Paris Nord, the Universit\'{e} de Paris, France, the French Ministry of Higher Education and Research [ENS.X19RDTME-SACLAY19-22] (C. R.), the French National Research Agency [ANR-18-CE17-0009-01 and ANR-18-CE37-0002-03] (A. D.-D.), [ANR-19-CE23-0026-04] (E. K.), and by the ``Digital health challenge'' Inserm CNRS joint program (C. R., E. K. and A. D.-D.).


\bibliographystyle{plainnat}
\bibliography{biblio.bib} 

\begin{thebibliography}{28}
\providecommand{\natexlab}[1]{#1}
\providecommand{\url}[1]{\texttt{#1}}
\expandafter\ifx\csname urlstyle\endcsname\relax
  \providecommand{\doi}[1]{doi: #1}\else
  \providecommand{\doi}{doi: \begingroup \urlstyle{rm}\Url}\fi

\bibitem[Auer(2003)]{auer2002using}
Peter Auer.
\newblock Using confidence bounds for exploitation-exploration trade-offs.
\newblock \emph{J. Mach. Learn. Res.}, 3\penalty0 (null):\penalty0 397–422,
  2003.
\newblock ISSN 1532-4435.

\bibitem[Brown and Patel(2017)]{brown2017standard}
Adam~S Brown and Chirag~J Patel.
\newblock A standard database for drug repositioning.
\newblock \emph{Scientific data}, 4:\penalty0 170029, 2017.

\bibitem[Bubeck and Cesa-Bianchi(2012)]{bubeckCB12survey}
S\'{e}astian Bubeck and Nicol\`{o} Cesa-Bianchi.
\newblock {Regret analysis of stochastic and nonstochastic multi-armed bandit
  problems}.
\newblock \emph{Fondations and Trends in Machine Learning}, 5(1):\penalty0
  1--122, 2012.

\bibitem[Bubeck et~al.(2013)Bubeck, Wang, and Viswanathan]{bubeck2013multiple}
S\'{e}bastian Bubeck, Tengyao Wang, and Nitin Viswanathan.
\newblock Multiple identifications in multi-armed bandits.
\newblock In Sanjoy Dasgupta and David McAllester, editors, \emph{Proceedings
  of the 30th International Conference on Machine Learning}, volume~28 of
  \emph{Proceedings of Machine Learning Research}, pages 258--265, Atlanta,
  Georgia, USA, 2013. PMLR.
\newblock URL \url{http://proceedings.mlr.press/v28/bubeck13.html}.

\bibitem[Bubeck et~al.(2009)Bubeck, Munos, and Stoltz]{bubeck2009pure}
S\'{e}bastien Bubeck, R\'{e}mi Munos, and Gilles Stoltz.
\newblock Pure exploration in multi-armed bandits problems.
\newblock In \emph{Proceedings of the 20th International Conference on
  Algorithmic Learning Theory}, ALT'09, page 23–37, Berlin, Heidelberg, 2009.
  Springer-Verlag.
\newblock ISBN 3642044131.

\bibitem[Chen et~al.(2017)Chen, Li, and Qiao]{chen2017nearly}
Lijie Chen, Jian Li, and Mingda Qiao.
\newblock {Nearly Instance Optimal Sample Complexity Bounds for Top-k Arm
  Selection}.
\newblock In Aarti Singh and Jerry Zhu, editors, \emph{Proceedings of the 20th
  International Conference on Artificial Intelligence and Statistics},
  volume~54 of \emph{Proceedings of Machine Learning Research}, pages 101--110,
  Fort Lauderdale, FL, USA, 2017. PMLR.
\newblock URL \url{http://proceedings.mlr.press/v54/chen17a.html}.

\bibitem[Cheng and Li(2016)]{cheng2016systematic}
Lijun Cheng and Lang Li.
\newblock Systematic quality control analysis of lincs data.
\newblock \emph{CPT: pharmacometrics \& systems pharmacology}, 5\penalty0
  (11):\penalty0 588--598, 2016.

\bibitem[Clark et~al.(2014)Clark, Hu, Feldmann, Kou, Chen, Duan, and
  Ma’ayan]{clark2014characteristic}
Neil~R Clark, Kevin~S Hu, Axel~S Feldmann, Yan Kou, Edward~Y Chen, Qiaonan
  Duan, and Avi Ma’ayan.
\newblock The characteristic direction: a geometrical approach to identify
  differentially expressed genes.
\newblock \emph{BMC bioinformatics}, 15\penalty0 (1):\penalty0 79, 2014.

\bibitem[Degenne and Koolen(2019)]{degenne19multiple}
R\'{e}my Degenne and Wouter~M Koolen.
\newblock Pure exploration with multiple correct answers.
\newblock In H.~Wallach, H.~Larochelle, A.~Beygelzimer, F.~d\textquotesingle
  Alch\'{e}-Buc, E.~Fox, and R.~Garnett, editors, \emph{Advances in Neural
  Information Processing Systems}, volume~32, pages 14591--14600. Curran
  Associates, Inc., 2019.
\newblock URL
  \url{https://proceedings.neurips.cc/paper/2019/file/60cb558c40e4f18479664069d9642d5a-Paper.pdf}.

\bibitem[Degenne et~al.(2020)Degenne, M\'{e}nard, Shang, and
  Valko]{degenne2020gamification}
R{\'e}my Degenne, Pierre M\'{e}nard, Xuedong Shang, and Micha\l{} Valko.
\newblock Gamification of pure exploration for linear bandits.
\newblock In Hal~Daum\'{e} III and Aarti Singh, editors, \emph{Proceedings of
  the 37th International Conference on Machine Learning}, volume 119 of
  \emph{Proceedings of Machine Learning Research}, pages 2432--2442. PMLR,
  13--18 Jul 2020.
\newblock URL \url{http://proceedings.mlr.press/v119/degenne20a.html}.

\bibitem[Delahaye-Duriez et~al.(2016)Delahaye-Duriez, Srivastava, Shkura,
  Langley, Laaniste, Moreno-Moral, Danis, Mazzuferi, Foerch, Gazina,
  et~al.]{delahaye2016rare}
Andr\'{e}e Delahaye-Duriez, Prashant Srivastava, Kirill Shkura, Sarah~R
  Langley, Liisi Laaniste, Aida Moreno-Moral, B{\'e}n{\'e}dicte Danis, Manuela
  Mazzuferi, Patrik Foerch, Elena~V Gazina, et~al.
\newblock Rare and common epilepsies converge on a shared gene regulatory
  network providing opportunities for novel antiepileptic drug discovery.
\newblock \emph{Genome biology}, 17\penalty0 (1):\penalty0 1--18, 2016.

\bibitem[Duan et~al.(2016)Duan, Reid, Clark, Wang, Fernandez, Rouillard,
  Readhead, Tritsch, Hodos, Hafner, et~al.]{duan2016l1000cds}
Qiaonan Duan, St~Patrick Reid, Neil~R Clark, Zichen Wang, Nicolas~F Fernandez,
  Andrew~D Rouillard, Ben Readhead, Sarah~R Tritsch, Rachel Hodos, Marc Hafner,
  et~al.
\newblock L1000cds 2: Lincs l1000 characteristic direction signatures search
  engine.
\newblock \emph{NPJ systems biology and applications}, 2:\penalty0 16015, 2016.

\bibitem[Fiez et~al.(2019)Fiez, Jain, Jamieson, and
  Ratliff]{fiez2019sequential}
Tanner Fiez, Lalit Jain, Kevin~G Jamieson, and Lillian Ratliff.
\newblock Sequential experimental design for transductive linear bandits.
\newblock In H.~Wallach, H.~Larochelle, A.~Beygelzimer, F.~d\textquotesingle
  Alch\'{e}-Buc, E.~Fox, and R.~Garnett, editors, \emph{Advances in Neural
  Information Processing Systems}, volume~32, pages 10667--10677. Curran
  Associates, Inc., 2019.
\newblock URL
  \url{https://proceedings.neurips.cc/paper/2019/file/8ba6c657b03fc7c8dd4dff8e45defcd2-Paper.pdf}.

\bibitem[Gabillon et~al.(2012)Gabillon, Ghavamzadeh, and
  Lazaric]{gabillon2012best}
Victor Gabillon, Mohammad Ghavamzadeh, and Alessandro Lazaric.
\newblock Best arm identification: A unified approach to fixed budget and fixed
  confidence.
\newblock In F.~Pereira, C.~J.~C. Burges, L.~Bottou, and K.~Q. Weinberger,
  editors, \emph{Advances in Neural Information Processing Systems}, volume~25,
  pages 3212--3220. Curran Associates, Inc., 2012.
\newblock URL
  \url{https://proceedings.neurips.cc/paper/2012/file/8b0d268963dd0cfb808aac48a549829f-Paper.pdf}.

\bibitem[Hoffman et~al.(2014)Hoffman, Shahriari, and
  Freitas]{hoffman2014correlation}
Matthew Hoffman, Bobak Shahriari, and Nando Freitas.
\newblock {On correlation and budget constraints in model-based bandit
  optimization with application to automatic machine learning}.
\newblock In Samuel Kaski and Jukka Corander, editors, \emph{Proceedings of the
  Seventeenth International Conference on Artificial Intelligence and
  Statistics}, volume~33 of \emph{Proceedings of Machine Learning Research},
  pages 365--374, Reykjavik, Iceland, 2014. PMLR.
\newblock URL \url{http://proceedings.mlr.press/v33/hoffman14.html}.

\bibitem[Hwang et~al.(2016)Hwang, Carpenter, Lauffenburger, Wang, Franklin, and
  Kesselheim]{hwang2016failure}
Thomas~J Hwang, Daniel Carpenter, Julie~C Lauffenburger, Bo~Wang, Jessica~M
  Franklin, and Aaron~S Kesselheim.
\newblock Failure of investigational drugs in late-stage clinical development
  and publication of trial results.
\newblock \emph{JAMA internal medicine}, 176\penalty0 (12):\penalty0
  1826--1833, 2016.

\bibitem[Kalyanakrishnan et~al.(2012)Kalyanakrishnan, Tewari, Auer, and
  Stone]{kalyanakrishnan2012pac}
Shivaram Kalyanakrishnan, Ambuj Tewari, Peter Auer, and Peter Stone.
\newblock Pac subset selection in stochastic multi-armed bandits.
\newblock In \emph{Proceedings of the 29th International Coference on
  International Conference on Machine Learning}, ICML'12, page 227–234,
  Madison, WI, USA, 2012. Omnipress.
\newblock ISBN 9781450312851.

\bibitem[Kauffman(1969)]{kauffman1969metabolic}
Stuart~A Kauffman.
\newblock Metabolic stability and epigenesis in randomly constructed genetic
  nets.
\newblock \emph{Journal of theoretical biology}, 22\penalty0 (3):\penalty0
  437--467, 1969.

\bibitem[Kaufmann(2014)]{kaufmann2014analyse}
\'{E}milie Kaufmann.
\newblock \emph{Analyse de strat\'{e}gies bay\'{e}siennes et fr\'{e}quentistes
  pour l'allocation s\'{e}quentielle de ressources}.
\newblock PhD thesis, 2014.
\newblock URL \url{http://www.theses.fr/2014ENST0056}.
\newblock Th\`{e}se de doctorat dirig\'{e}e par Capp\'{e}, Olivier et Garivier,
  Aur\'{e}lien, Signal et images Paris, ENST 2014.

\bibitem[Kaufmann and Kalyanakrishnan(2013)]{kaufmann2013information}
\'{E}milie Kaufmann and Shivaram Kalyanakrishnan.
\newblock Information complexity in bandit subset selection.
\newblock In Shai Shalev-Shwartz and Ingo Steinwart, editors, \emph{Proceedings
  of the 26th Annual Conference on Learning Theory}, volume~30 of
  \emph{Proceedings of Machine Learning Research}, pages 228--251, Princeton,
  NJ, USA, 2013. PMLR.
\newblock URL \url{http://proceedings.mlr.press/v30/Kaufmann13.html}.

\bibitem[Musa et~al.(2018)Musa, Ghoraie, Zhang, Glazko, Yli-Harja, Dehmer,
  Haibe-Kains, and Emmert-Streib]{musa2018review}
Aliyu Musa, Laleh~Soltan Ghoraie, Shu-Dong Zhang, Galina Glazko, Olli
  Yli-Harja, Matthias Dehmer, Benjamin Haibe-Kains, and Frank Emmert-Streib.
\newblock A review of connectivity map and computational approaches in
  pharmacogenomics.
\newblock \emph{Briefings in bioinformatics}, 19\penalty0 (3):\penalty0
  506--523, 2018.

\bibitem[R\'{e}da and Wilczy\'{n}ski(2020)]{reda2019automated}
Cl\'{e}mence R\'{e}da and Wilczy\'{n}ski.
\newblock Automated inference of gene regulatory networks using explicit
  regulatory modules.
\newblock \emph{Journal of Theoretical Biology}, 486:\penalty0 110091, 2020.
\newblock ISSN 0022-5193.
\newblock \doi{https://doi.org/10.1016/j.jtbi.2019.110091}.
\newblock URL
  \url{http://www.sciencedirect.com/science/article/pii/S0022519319304606}.

\bibitem[Soare et~al.(2014)Soare, Lazaric, and Munos]{soare14BAIlin}
Marta Soare, Alessandro Lazaric, and R\'{e}mi Munos.
\newblock Best-arm identification in linear bandits.
\newblock In Z.~Ghahramani, M.~Welling, C.~Cortes, N.~Lawrence, and K.~Q.
  Weinberger, editors, \emph{Advances in Neural Information Processing
  Systems}, volume~27, pages 828--836. Curran Associates, Inc., 2014.
\newblock URL
  \url{https://proceedings.neurips.cc/paper/2014/file/f387624df552cea2f369918c5e1e12bc-Paper.pdf}.

\bibitem[Subramanian et~al.(2017)Subramanian, Narayan, Corsello, Peck, Natoli,
  Lu, Gould, Davis, Tubelli, Asiedu, et~al.]{subramanian2017next}
Aravind Subramanian, Rajiv Narayan, Steven~M Corsello, David~D Peck, Ted~E
  Natoli, Xiaodong Lu, Joshua Gould, John~F Davis, Andrew~A Tubelli, Jacob~K
  Asiedu, et~al.
\newblock A next generation connectivity map: L1000 platform and the first
  1,000,000 profiles.
\newblock \emph{Cell}, 171\penalty0 (6):\penalty0 1437--1452, 2017.

\bibitem[Szklarczyk et~al.(2016)Szklarczyk, Morris, Cook, Kuhn, Wyder,
  Simonovic, Santos, Doncheva, Roth, Bork, et~al.]{szklarczyk2016string}
Damian Szklarczyk, John~H Morris, Helen Cook, Michael Kuhn, Stefan Wyder, Milan
  Simonovic, Alberto Santos, Nadezhda~T Doncheva, Alexander Roth, Peer Bork,
  et~al.
\newblock The string database in 2017: quality-controlled protein--protein
  association networks, made broadly accessible.
\newblock \emph{Nucleic acids research}, page gkw937, 2016.

\bibitem[Tarjan(1972)]{tarjan1972depth}
Robert Tarjan.
\newblock Depth-first search and linear graph algorithms.
\newblock \emph{SIAM journal on computing}, 1\penalty0 (2):\penalty0 146--160,
  1972.

\bibitem[Xu et~al.(2018)Xu, Honda, and Sugiyama]{xu2017fully}
Liyuan Xu, Junya Honda, and Masashi Sugiyama.
\newblock A fully adaptive algorithm for pure exploration in linear bandits.
\newblock In Amos Storkey and Fernando Perez-Cruz, editors, \emph{Proceedings
  of the Twenty-First International Conference on Artificial Intelligence and
  Statistics}, volume~84 of \emph{Proceedings of Machine Learning Research},
  pages 843--851. PMLR, 2018.
\newblock URL \url{http://proceedings.mlr.press/v84/xu18d.html}.

\bibitem[Yordanov et~al.(2016)Yordanov, Dunn, Kugler, Smith, Martello, and
  Emmott]{yordanov2016method}
Boyan Yordanov, Sara-Jane Dunn, Hillel Kugler, Austin Smith, Graziano Martello,
  and Stephen Emmott.
\newblock A method to identify and analyze biological programs through
  automated reasoning.
\newblock \emph{NPJ systems biology and applications}, 2:\penalty0 16010, 2016.

\end{thebibliography}

\appendix

\onecolumn

\section{COMPARING COMPLEXITY CONSTANTS}\label{sec:complexity_constants}

\subsection{About the classic instance for linear Top-$m$ identification}

In higher dimensions, and when $m=K-2$, $\omega$ can be seen as the angle between the $m+1$-best arm vector and the hyperplane formed by the $m$ best arm feature vectors. In order to check if, for $m \geq 1$, decreasing the value of $\omega \in (0, \frac{\pi}{2})$ yields to harder instances (as it is for $m=1$), we ran the bandit algorithms on the instance $K=4$, $N=3$, $m=2$ for $\omega \in \{\frac{\pi}{3}, \frac{\pi}{6}\}$. The resulting boxplots are shown in Figure~\ref{fig:boxplots_supplementary}. It can then be seen that indeed, for all algorithms, the empirical average sample complexity increases as $\omega$ decreases, which is an argument in favour of the use of this type of instance for the test of linear Top-$m$ algorithms.

\begin{figure}[ht]
\includegraphics[scale=0.3]{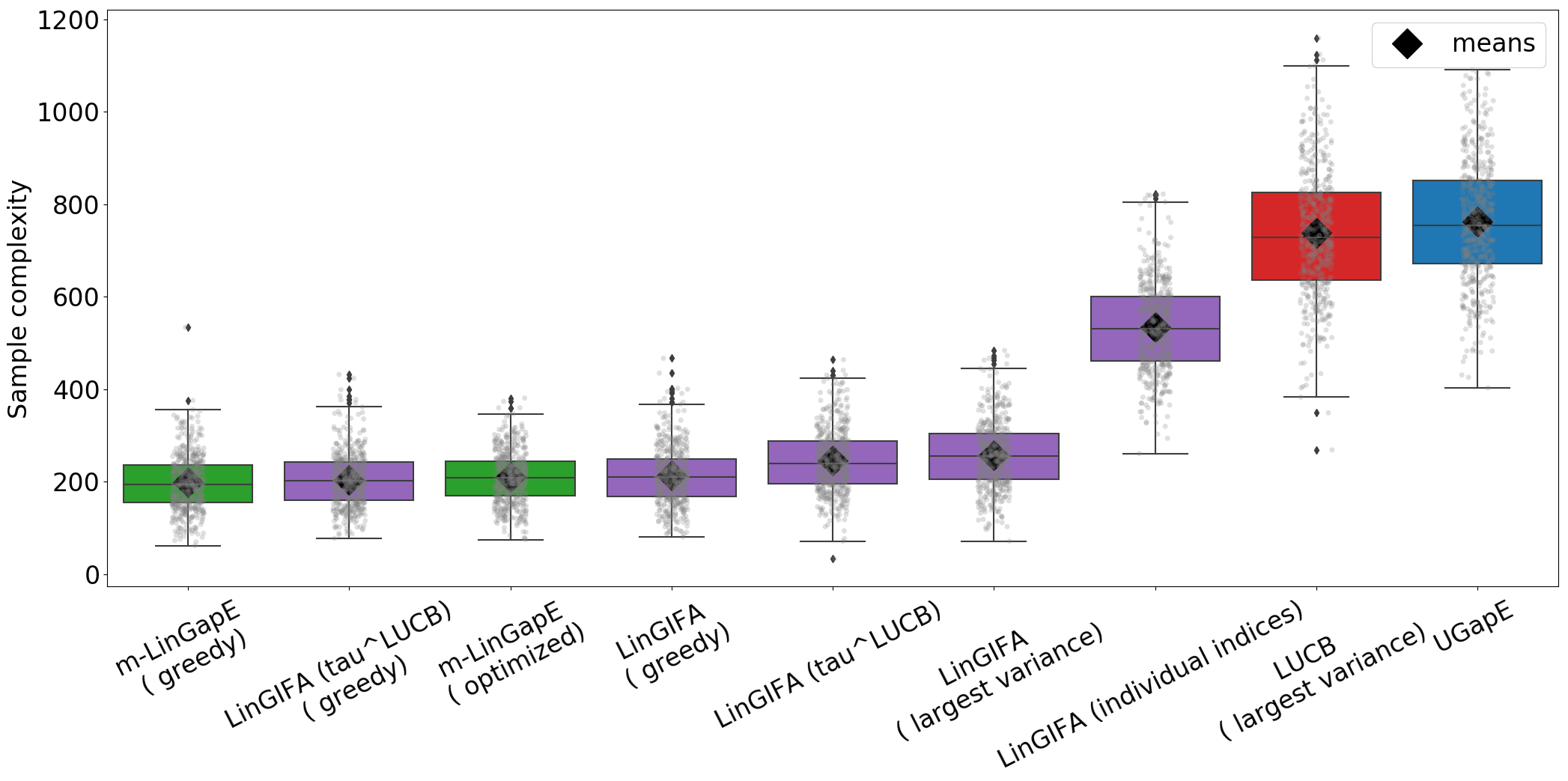}\\
\includegraphics[scale=0.3]{classic_K4_N3_m2_gaussian_pi6.png}
\caption{From top to bottom: classic instances (a) $K=4$, $\omega=\frac{\pi}{3}$, $m=2$~; (b) $K=4$, $\omega=\frac{\pi}{6}$, $m=2$. Error frequencies are rounded up to $5$ decimal places.}
\label{fig:boxplots_supplementary}
\end{figure}

\subsection{Comparing complexity constants}

Below is the table to which we refer to in Section~\ref{sec:discussion_exp}.

\begin{table*}[ht]
    \centering
    \caption{Comparison of complexity constants in $m$-LinGapE and UGapE ($\%$ on $1,000$ random instances).}
    \label{tab:comparison_complexity}
    \begin{tabular}{l|c|c|c|c|c|c|c|c|c|c|c|c|c|c|c}
	\textbf{$D$} & $0.25$ & $0.5$ & $0.25$ & $0.5$ & $0.25$ & $0.5$ & $0.25$ & $0.25$ & $0.5$ & $0.25$ & $0.5$ & $0.25$ & $0.5$ & $0.25$ & $0.5$\\
	\hline
	\textbf{$K$} & $10$ & $10$ & $10$ & $10$ & $10$ & $10$ & $20$ & $20$ & $20$ & $20$ & $20$ & $30$& $30$& $30$& $30$\\
	\hline
	\textbf{$m$} & $4$ & $4$ & $4$ & $4$ & $4$ & $4$ & $7$ & $7$ & $7$ & $7$ & $7$ & $11$& $11$& $11$& $11$\\
	\hline
	\textbf{$N$} & $5$ & $5$ & $10$ & $10$ & $20$ & $20$ & $10$ & $20$ & $20$ & $40$ & $40$ & $15$& $15$& $30$& $30$\\
	\hline
	\textbf{\%} & $29.1\%$ & $30.8\%$ & $0.0\%$ & $0.0\%$ & $0.0\%$ & $0.0\%$ & $0.6\%$ & $0.0\%$ & $0.0\%$ & $0.0\%$ & $0.0\%$ & $0.1\%$& $0.1\%$& $0.0\%$& $0.0\%$\\
    \end{tabular}
\end{table*}

We have tested if, empirically, LinGIFA was more performant than LinGapE on instances where $\HA{\text{$m$-LinGapE(2)}} \leq \HA{\text{UGapE}}$, since LinGIFA has a similar structure as UGapE. We generated a random linear instance, following the procedure described in Section~\ref{sec:discussion_exp} in the paper, with $K=10$, $N=5$, $D=0.5$. For $m=3$, the condition $\HA{\text{$m$-LinGapE(2)}} \leq \HA{\text{UGapE}}$ is satisfied, whereas it is not when $m=8$. We considered Gaussian reward distributions. See Figure~\ref{fig:boxplots_complexity}. From these results, we notice that both algorithms are actually similar in sample complexity in both instances. Hence, even if the condition $\HA{\text{$m$-LinGapE(2)}} \leq \HA{\text{UGapE}}$ is seldom satisfied as seen in Table~\ref{tab:comparison_complexity}, in practice, $m$-LinGapE with the optimized rule is still performant.

\begin{table*}[ht]
    \centering
    \caption{Values of complexity constants in $m$-LinGapE and UGapE on the randomly generated instance.}
    \begin{tabular}{l|c|c}
	& $m=3$ & $m=8$\\
	\hline
	$\HA{\text{$m$-LinGapE(2)}}$ & $4,545.97$ & $32,124.01$\\
	$\HA{\text{UGapE}}$ & $5,047.76$ & $27,622.18$\\
	$\mu_m-\mu_{m+1}$ & $0.075$ & $0.029$\\
	$\HA{\text{$m$-LinGapE(2)}} \leq \HA{\text{UGapE}}?$ & \textbf{True} & \textbf{False}\\
    \end{tabular}
\end{table*}

\begin{figure}[ht]
\includegraphics[scale=0.3]{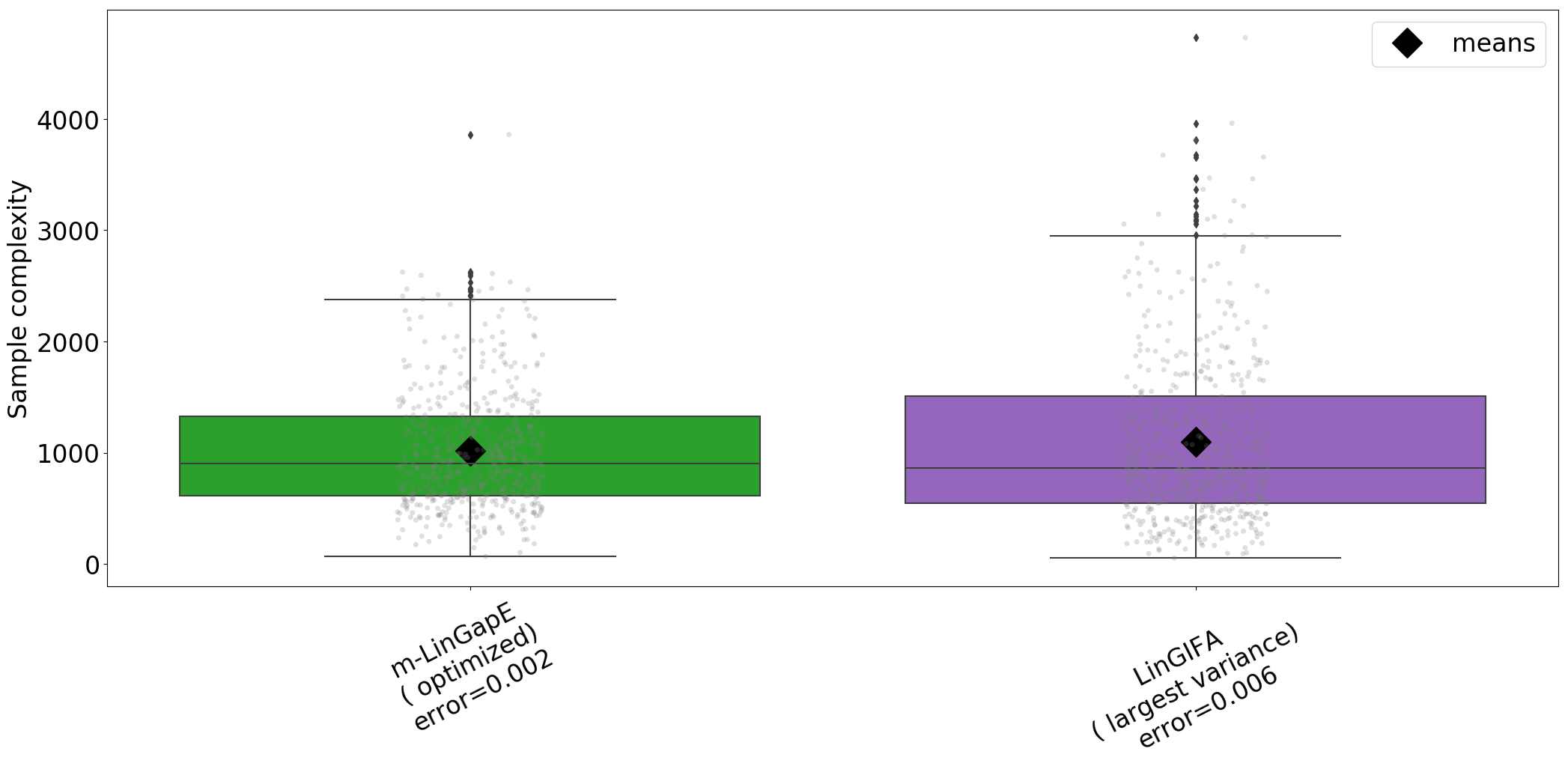}\\
\includegraphics[scale=0.3]{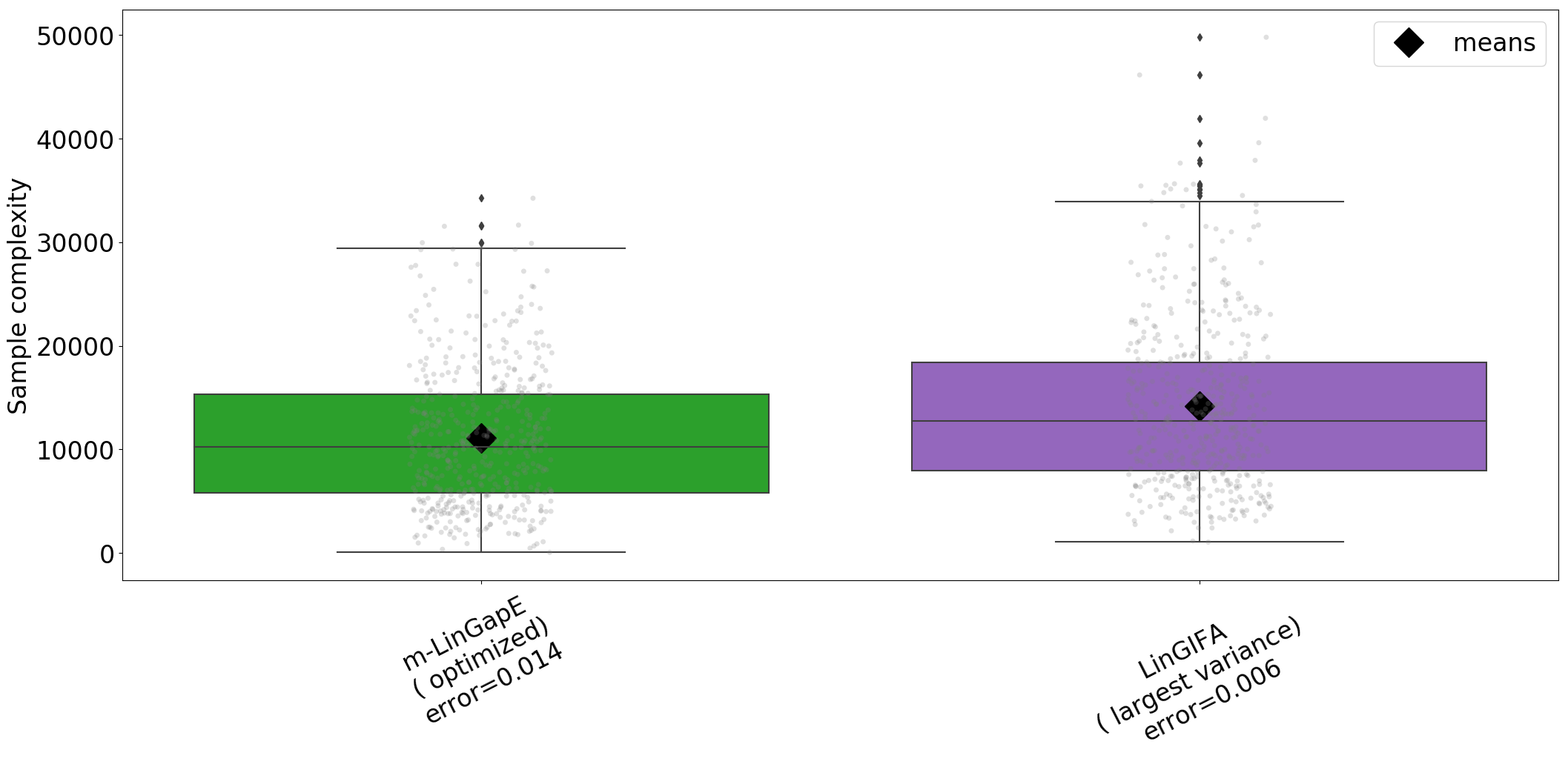}
\caption{From top to bottom: $m=3$, $m=8$. Error frequencies are rounded up to $5$ decimal places.}
\label{fig:boxplots_complexity}
\end{figure}

\section{DRUG REPURPOSING INSTANCE}\label{sec:dr_instance}

Remember that we call ``phenotypes'' gene activity profiles of patients and controls (healthy group) (that is, vectors which represent the genewise activity in a finite set of genes). We focus on a finite set of genes called M$30$, which has been shown to have a global gene activity that is anti-correlated to epileptic gene activity profiles~\citep{delahaye2016rare}. The bandit instance comprises of arms/drugs, which, once pulled, return a single score/reward which quantifies their ability to ``reverse'' a epileptic phenotype --that is, an anti-epileptic treated epileptic phenotype should be closer to a healthy phenotype. A gene regulatory network (GRN) is a summary of gene transcriptomic interactions as a graph: nodes are genes or proteins, (directed) edges are regulatory interactions. We see a GRN as a Boolean network (BN): nodes can have two states ($0$ or $1$), and each of them is assigned a so-called ``regulatory function'', that is, a logical formul\ae ~which updates their state given the states of regulators (i.e., predecessors in the network) at each time step. In order to infer the effect of a treatment, one can set as initial network state (``initial condition'') the patient phenotype masked by the perturbations on the drug targets, and iteratively update the network state until reaching an attractor state (\texttt{phenotype\_prediction} procedure).

\paragraph{Building the Boolean network}

We use the Boolean network inference method described in~\citep{yordanov2016method} using code at repository \texttt{https://github.com/regulomics/expansion-network}~\citep{reda2019automated}. We get the unsigned undirected regulatory interactions from the \emph{protein-protein interaction network} (PPI) of M$30$, using the STRING database~\citep{szklarczyk2016string}. Using expression (or, as we called it in the paper, gene activity) data in the hippocampus from UKBEC data (Gene Expression Omnibus (GEO) accession number $GSE46706$), a Pearson's R correlation matrix is computed, which allows signing the interactions using pairwise correlation signs. Then, considering that the effects of a gene perturbation can only be seen on connected nodes, we only keep strongly connected components in which at least one gene perturbation in LINCS L$1000$ experiments occurs, using Tarjan's algorithm~\citep{tarjan1972depth}.

Then, in order to direct the edges in the network using the inference method, we restrict the experiments extracted from LINCS L1000 to those in SH-SY5Y human neuron cells (neuroblastoma from bone marrow), with a positive interference scale score~\citep{cheng2016systematic}, which quantify the success of the perturbation experiment. For each experiment (knockdown via shRNA, overexpression via cDNA) on a gene in M30 in this cell line, we extract from Level 3 LINCS same-plate untreated, genetic control and perturbed profiles (each of them being real-valued vectors of size $\approx 100$, the number of genes in M$30$) such as the perturbed profile on this plate has the largest value of \textit{distil\_ss} which quantifies experimental replication. This procedure yields a total of ($1$ untreated + $2$ replicates of genetic control + $2$ replicates of perturbed) $\times 3$ experimental profiles. Each experimental constraint for the GRN inference is defined as follows ($G=101$ is the number of $M30$ genes in the network):

\begin{itemize}
\item \textbf{Initial condition}: Untreated profile which has been binarized using the \texttt{binarize\_via\_histogram} procedure (peak detection in histogram of gene expression values using persistent topology). Value: $\{0,1,\bot\}^{G}$. The number of non-$\bot$ values is $66$.
\item \textbf{Perturbation}: Gene-associated value is equal to $1$ if and only if the gene is perturbed in the experiment.
\item \textbf{Final/fixpoint condition}: Vector which has been obtained by running Characteristic Direction~\citep{clark2014characteristic} (CD) on $[\text{treated}||\text{genetic control}]$ (in the call to the function, treated profiles were annotated $2$ whereas genetic control ones were annotated $1$) profiles, which yields a vector in $\{0,1,\bot\}^{G}$, where $0$'s (resp. $1$'s) are significantly down- (resp. up-) regulated genes in treated profiles compared to control ones. Note that $[P_1||P_2]$ means that we compute the genewise activity change from group $P_2$ to group $P_1$ (hence, this is not symmetrical). The number of non-$\bot$ values is around $22$.
\end{itemize}

The inferred GRN should satisfy all experimental constraints by assigning logical functions to genes and selecting gene interactions. This network is displayed in Figure~\ref{fig:inferred_grn}.

\begin{table*}[ht]
\centering
\caption{Experiments in SH-SY5Y human neuron-like cell line for GRN inference. Inference parameter: $t=20$, asynchronous dynamics.}
\label{tab:experiments}
\begin{tabular}{c|c|c}
\textbf{Perturbed gene} & \textbf{Experiment type} & \textbf{Exposure time}\\
\hline
CACNA1C & KD & 120 h\\
CDC42 & KD & 120 h\\
KCNA2 & KD & 120 h\\
\end{tabular}
\end{table*}

\begin{figure}[ht]
\centering
\includegraphics[scale=0.29]{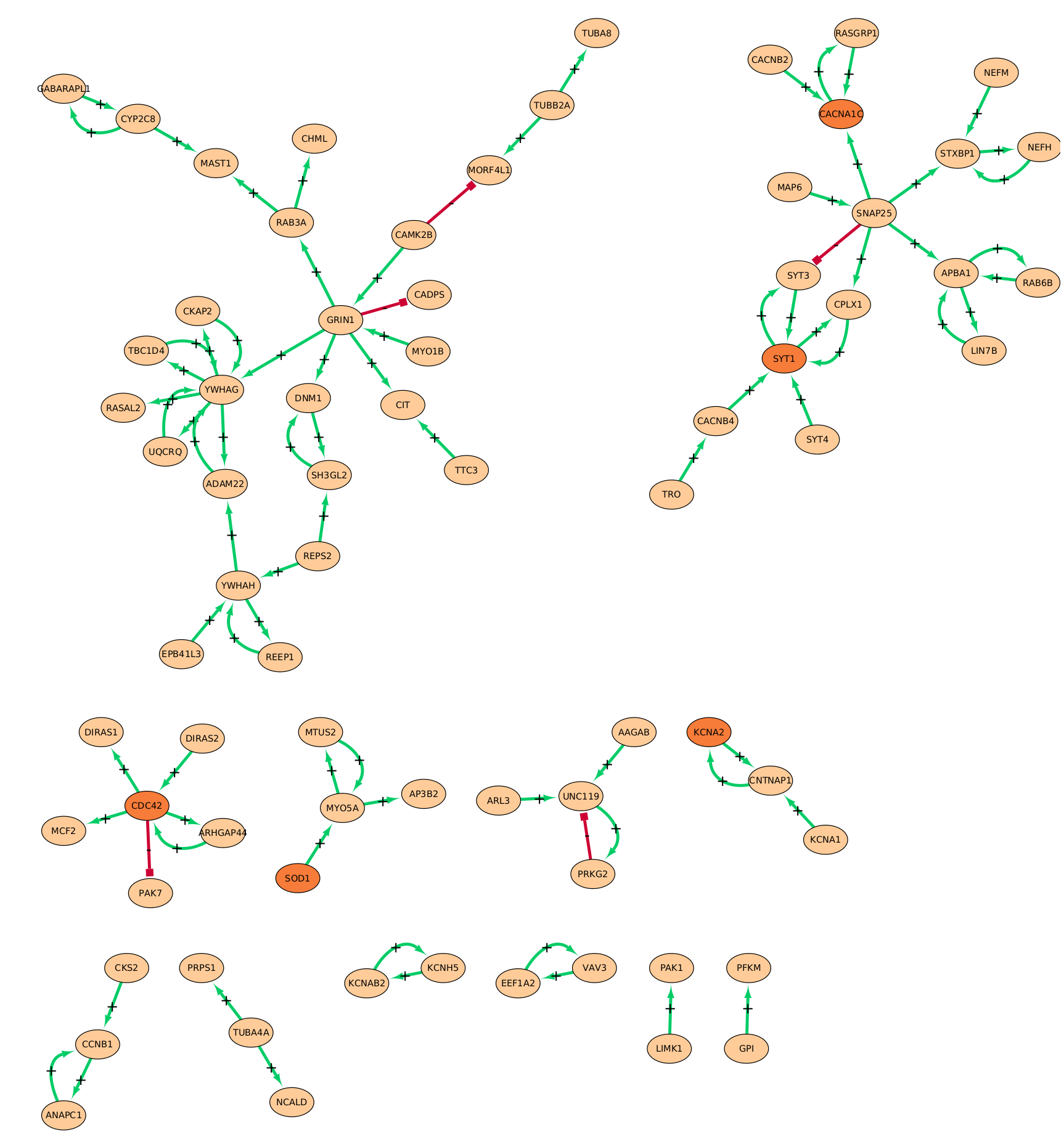}
\caption{Inferred GRN for the drug repurposing instance on epilepsy. Genes present in this network belong to the M$30$ set. Green edges labelled ``+'' (resp. red edges labelled ``-'') are activatory (resp. inhibitory) interactions from one regulatory gene on a target gene, that is, that increase (resp. decrease) target gene activity. Deep orange nodes are perturbed nodes in the SHSY5Y cell line in LINCS L$1000$.}
\label{fig:inferred_grn}
\end{figure}

\paragraph{Getting the arm/drug features}

The features we use are the drug signatures ($K=509$ in the binary drug signature dataset). Given a drug, we compute them as follows:

\begin{enumerate}
	\item First, in the Level $3$ LINCS L1000 database~\citep{subramanian2017next}, we select the cell line with the highest transcriptomic activity score, or TAS (quantifying the success of treatment in this specific cell line: we expect to obtain more reproducible experiments in this cell line if this score is high). Then, considering experiment in this cell line, treatment by the considered drug, we select the \textit{brew\_prefix} (identifier for experimental plate) which correspond to the treated expression profile with the largest value of \textit{distil\_ss} (which quantifies the reproducibility of the profile across replicates), and we get the corresponding same-plate control (vehicle) profile. We also get one same-plate replicate of the considered treated experiment and another of the control experiment (total number of profiles: $2+2$). 
	\item We apply on this set of profiles Characteristic Direction~\citep{clark2014characteristic} in order to get the relative genewise expression change due to treatment from control sample group $CD([\text{treated}||\text{vehicle control}])$. This yields a real-valued vector in $[-1, 1]^{G}$ which will be used in the baseline method L1000 CDS$^2$, and a binary vector in $\{0,1,\bot\}^{G}$ in our scoring method, which is the so-called drug signature.
\end{enumerate}

\paragraph{Epileptic patient/control phenotypes}

We fetched data from GEO accession number $GSE77578$, which was then quantile-normalized across all patient ($|P_p|=18$) and control ($|P_c|=17$) samples. We run Characteristic Direction~\citep{clark2014characteristic} $CD([P_c||P_p])$ in order to get the ``differential phenotype'' from controls to patients, which is the way we chose in order to aggregate control profiles and only considering differentially expressed genes.

\paragraph{Drug ``true'' scores}

We get them from RepoDB~\citep{brown2017standard} database, which is a curated version from \texttt{clinicaltrials.gov}, and from literature. To each drug is associated an integer: $1$ if the drug is antiepileptic, $0$ if it is unknown, $-1$ if it is a proconvulsant drug.
\vspace{-1cm}
\paragraph{Masking procedure $\rtimes$}

We use this (asymmetric) function in order to generate the initial condition from which an attractor state, if it exists, should be fetched: $(x \rtimes y)[j] = y[j]$ if $y[j] \in \{0,1\}$ else $x[j]$. This aims at mimicking the immediate effect of treatment on gene activity.
\vspace{-1.2cm}
\paragraph{Running the simulator \emph{via} the GRN}

Given collected patient and control phenotypes, and seeing arms as potentially repurposed drugs, the procedure to generate a reward from a given arm $a$ is as described in Algorithm~\ref{alg:grn_simulator}. We compare this method to a simpler signature reversion method, used in the web application L1000 CDS$^2$~\citep{duan2016l1000cds}, which is deterministic, and compares directly drug signatures and differential phenotypes. The full procedure is described in Algorithm~\ref{alg:cosine_baseline}. We have tested our method on a subset of drugs with respect to this baseline. The results can be seen in Figure~\ref{fig:test_to_baseline}. 

Note that returning a score for a single drug with our method is usually a matter of a few minutes, but the computation time can drastically increase when considering a higher number of nodes in the Boolean network, so that is why, even if on this instance we could run all computations for each drug and for each initial patient sample for drug repurposing (which is what we do in Figure~\ref{fig:test_to_baseline} anyway in order to check that our method yields correct results with respect to known therapeutic indications), we think this model is interesting to test our bandit algorithms.

Moreover, the linear dependency between features and scores does not hold: indeed, in our subset of $K=10$ arms, computing the least squares estimate of $\theta$ using the mean rewards $\overline{m}$ as true values, and denoting $X$ the concatenation of drug signatures, that is, $\theta = (X^\top X)^{-1}X\overline{m}$, gives a high value of $\|\theta-\overline{m}\| \approx 20.9$. The linear setting in bandits is simply the easiest contextual setting to analyze. Although this non-linearity, along with the fact that the initial condition is randomized, might be the main reason why the empirical sample complexity in our subset instance is a lot higher than $10 \times 18$ for all algorithms, even if linear algorithms are noticeably more performant than classical ones.
\vspace{-1.2cm}
\begin{algorithm}
\begin{algorithmic}
\STATE \textbf{requires} $G$ GRN, phenotypes of patient (diseased) and control (healthy) individuals w.r.t. a given disease $P_p \in [0, 15]^{G}$, $P_c \in [0, 15]^{G}$, $a$ arm/drug to be tested, with binary drug signature $s^b_{a} \in \{0,1,\bot\}^{G}$.
\STATE \# \textbf{differential phenotype is computed: controls$||$patients}
\STATE $D \in \{0,1,\bot\}^{G} = \emph{CD}([P_c || P_p])$
\STATE \# \textbf{patient phenotype is uniformly sampled from the pool of patient phenotypes}
\STATE $p \thicksim \cU(P_p)$
\STATE $p^b \in \{0,1,\bot\}^{G} \leftarrow \texttt{binarize\_via\_histogram}(p)$
\STATE $p^r \in \{0,1\}^{G} \leftarrow$ \texttt{phenotype\_prediction}(GRN=G, initial\_condition=$(p^b \rtimes s^b_{a}) \in \{0,1,\bot\}^{G}$)
\STATE \# \textbf{comparison function \texttt{cosine\_score} is run on the intersection of supports of $D$ and $p^r$}
\STATE \# \textbf{this intersection is equal to $50$ in practice, which is the size of the support of $D$}
\STATE $r \leftarrow \texttt{cosine\_score}_{|D|\cap|p^r|}(D, p^r)$
\STATE \textbf{returns} $r$
\end{algorithmic}
\caption{Reward generation via the ``GRN simulator''.}
\label{alg:grn_simulator}
\end{algorithm}
\vspace{-3cm}
\begin{algorithm}
\begin{algorithmic}
\STATE \textbf{requires} Phenotypes of patient (diseased) and control (healthy) individuals w.r.t. a given disease $P_p \in [0, 15]^{G}$, $P_c \in [0, 15]^{G}$, $a$ arm/drug to be tested, with non-binary, full signature $s_{a}  \in [-1,1]^{G}$.
\STATE \# \textbf{differential phenotype is computed: patients$||$controls}
\STATE $C \in \{0,1,\bot\}^{G} = \emph{CD}([P_p || P_c])$
\STATE $r \leftarrow 1-\texttt{cosine\_score}(C, s_{a})$
\STATE \textbf{returns} $r$
\end{algorithmic}
\caption{Reward via baseline method from L1000 CDS$^2$~\citep{duan2016l1000cds}.}
\label{alg:cosine_baseline}
\end{algorithm}

\begin{figure}[H]
\centering
\includegraphics[scale=0.25]{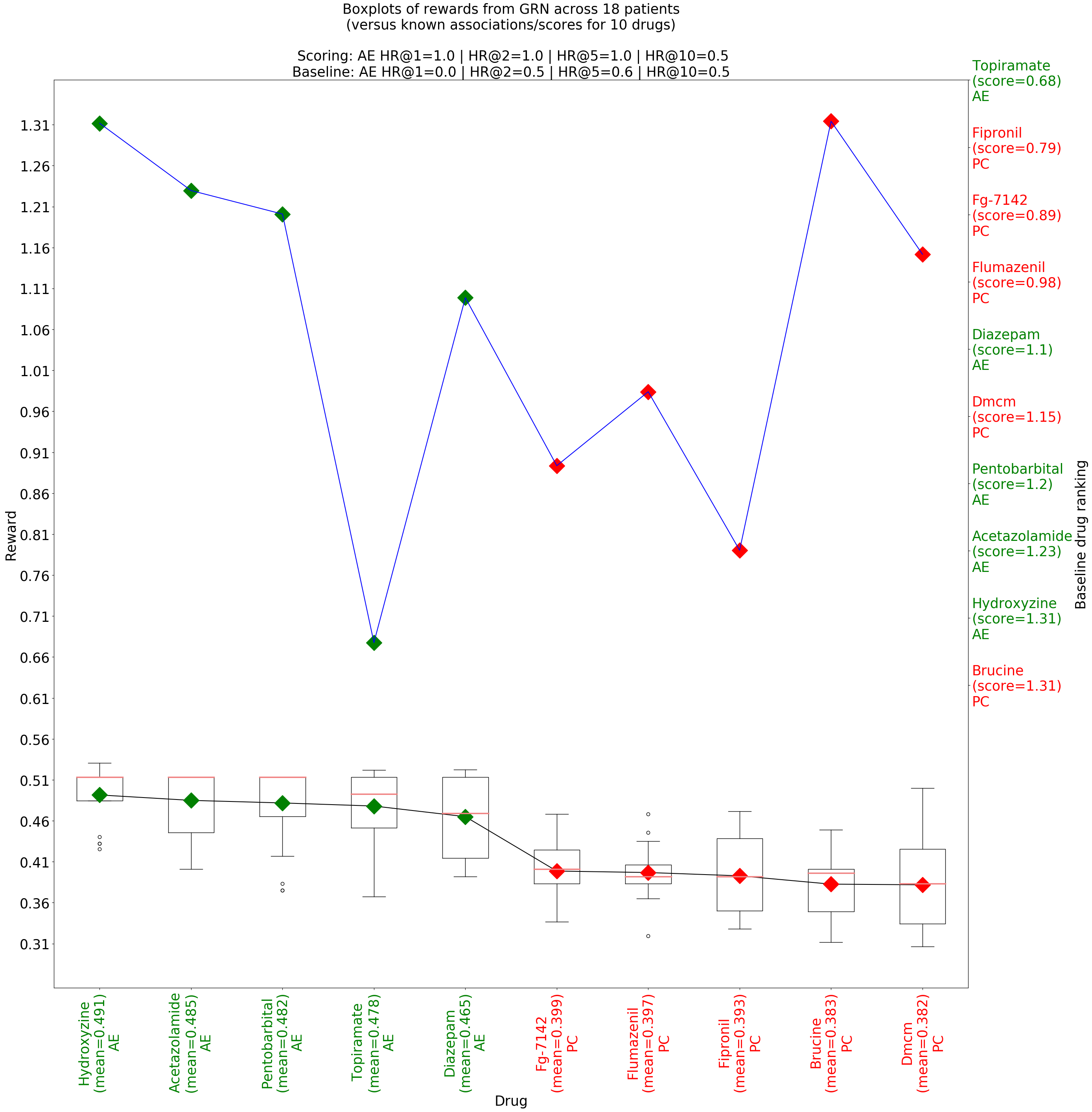}
	\caption{We consider a subset of drugs of size $10$ ($5$ with positive association score, $5$ with negative score), which is the one tested in the paper. For validation, we plot a boxplot of the rewards obtained for each initial patient sample, for each drug. Mean is colored as green if the drug is antiepileptic (AE), resp. red if it is proconvulsant (PC), with the corresponding drug name (in red if its true score is negative, in green if it is positive). The baseline score is plot in in blue. For both methods, the highest the score is, the better (the ``more'' anti-epileptic the drug is predicted). We computed and reported above the plot the Hit Score at rank $r$ (HR$@r$), that is, the mean accuracy on the class AE on the Top-$r$ scores, for $r \in \{1,2,5,10\}$ for each of the methods (Scoring or Baseline).}
\label{fig:test_to_baseline}
\end{figure}

\begin{figure}[ht]
\includegraphics[scale=0.35]{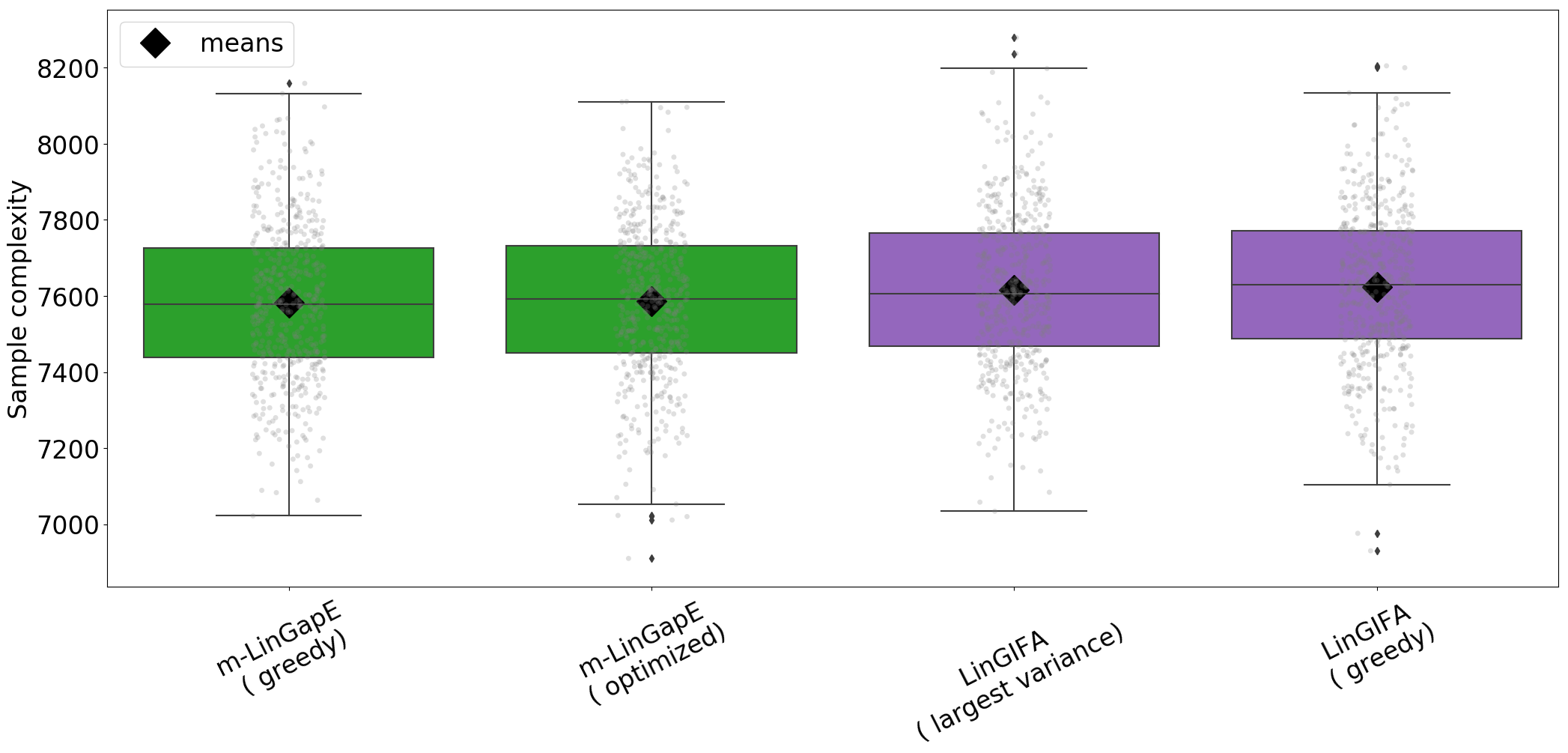}
\caption{Drug repurposing instance $K=10$, $500$ simulations, $m=5$, $\delta=0.05,\epsilon=0,\sigma=0.5,\lambda=\sigma/20$. Close-up from Figure~\ref{fig:boxplots_a} from the paper.}
\label{fig:dr_boxplots_closeup}
\end{figure}

\clearpage

\section{UPPER BOUNDS FOR GIFA ALGORITHMS}~\label{sec:upper_bounds}

\begin{lemma}
	In algorithm $m$-LinGapE, for any selection rule, on event $\cE \triangleq \INTERSEC{t > 0} \INTERSEC{i,j \in \ARMS} \Big(\GAPPAIRED{i}{j} \in [-B_{j,i}(t), B_{i,j}(t)]\Big)$, for all $t > 0$, $B_{c_t, b_t}(t) \leq \min(-({\GAP{{b_t}}} \lor {\GAP{{c_t}}})+2W_t(b_t,c_t), 0)+W_t(b_t,c_t)$. \textnormal{(Lemma~\ref{lemma:m-LinGapE_bound} in the paper)}
\end{lemma}

\begin{proof}
Let us use two properties:

\paragraph{1.} As $b_t \in J(t)$ and $c_t \notin J(t)$, it holds in particular that $\EMPMU{b_t}{t} \geq \EMPMU{c_t}{t}$, hence $B_{c_t,b_t}(t) = \EMPGAP{c_t}{b_t}{t} + W_t(b_t,c_t) \leq W_t(b_t,c_t)$.
\paragraph{2.} From the definitions of $b_t \text{ and } c_t$, it holds that $B_{c_t, b_t}(t)  = \max_{j \in J(t)} \max_{i \notin J(t)} B_{i,j}(t)$.

Property $1$ already establishes that $B_{c_t, b_t}(t) \leq W_t(b_t,c_t)$, it therefore remains to show that 
$B_{c_t, b_t}(t)  \leq -({\GAP{{b_t}}} \lor {\GAP{{c_t}}})+3W_t(b_t,c_t)$. We do it by distinguishing four cases:

\paragraph{(i)} \textbf{$b_t \in \TOPM$ and $c_t \notin \TOPM$}: In that case ${\GAP{{b_t}}} = \mu_{b_t} - \mu_{m+1}$ and ${\GAP{{c_t}}} = \mu_{m} - \mu_{c_t}$. As event $\cE$ holds, one has $B_{c_t,b_t}(t) = -B_{b_t,c_t}(t)+2W_t(b_t,c_t) \leq \GAPPAIRED{c_t}{b_t}+2W_t(b_t,c_t)$. As $c_t \notin \TOPM$, $\mu_{c_t} \leq \mu_{m+1}$, and $\GAPPAIRED{c_t}{b_t} \leq  \mu_{m+1} - \mu_{b_t} = -{\GAP{{b_t}}}$. But as $b_t \in \TOPM$, it also holds that $\mu_{b_t} \geq \mu_{m}$, and $\GAPPAIRED{c_t}{b_t} \leq  \mu_{c_t} - \mu_{m} = -{\GAP{{c_t}}}$. Hence $B_{c_t, b_t}(t)  \leq -({\GAP{{b_t}}} \lor {\GAP{{c_t}}})+2W_t(b_t,c_t) \leq -({\GAP{{b_t}}} \lor {\GAP{{c_t}}})+3W_t(b_t,c_t)$.

\paragraph{(ii)} \textbf{$b_t \not\in \TOPM$ and $c_t \in \TOPM$}: Using Property $1$:
    \begin{eqnarray*}
        B_{c_t, b_t}(t) & \leq & W_t(b_t,c_t) \leq \EMPGAP{b_t}{c_t}{t} + W_t(b_t,c_t) = B_{b_t,c_t}(t) = -B_{c_t,b_t}(t)+2W_t(b_t,c_t) \leq \GAPPAIRED{b_t}{c_t} + 2W_t(b_t,c_t)\\
     \end{eqnarray*}
as event $\cE$ holds. One can show with the same arguments as in the previous case that $B_{c_t, b_t}(t) \leq -({\GAP{{b_t}}} \lor {\GAP{{c_t}}})+3W_t(b_t,c_t)$. 

\paragraph{(iii)} \textbf{$b_t \not\in \TOPM$ and $c_t \not\in \TOPM$}: In that case, there must exist $b \in \TOPM$ that belongs to $J(t)^c$. From the definition of $c_t$, it follows that $B_{c_t,b_t}(t) \geq B_{b, b_t}(t)$. Hence, using furthermore Property $1$, the definition of $c_t$, event $\cE$ and $b \in \TOPM$, $W_t(b_t,c_t) \geq B_{c_t, b_t}(t) \geq B_{b,b_t}(t) \geq \GAPPAIRED{b}{b_t} \geq \GAPPAIRED{m}{b_t} = {\GAP{{b_t}}}$. It follows that, using event $\cE$:
    \begin{eqnarray*}
         B_{c_t, b_t}(t) & \leq & \GAPPAIRED{c_t}{b_t} + 2W_t(b_t,c_t) = (\mu_{c_t} - \mu_{m}) + (\mu_{m} - \mu_{b_t}) + 2W_t(b_t,c_t) =  - {\GAP{{c_t}}} + {\GAP{{b_t}}} + 2W_t(b_t,c_t)\\
         & & \mbox{ ($b_t \not\in \TOPM$ and $c_t \not\in \TOPM$)}\\
         & \leq & - {\GAP{{c_t}}} + 3 W_t(b_t,c_t)\\
    \end{eqnarray*}
And it also holds by Property $1$ that:
    \begin{eqnarray*}
        B_{c_t, b_t}(t) & \leq & W_t(b_t,c_t) = - W_t(b_t,c_t) + 2 W_t(b_t,c_t) \\
        & \leq & -{\GAP{{b_t}}} + 2W_t(b_t,c_t) \leq -{\GAP{{b_t}}} + 3 W_t(b_t,c_t)\\
    \end{eqnarray*}
Hence $B_{c_t, b_t}(t) \leq -({\GAP{{b_t}}} \lor {\GAP{{c_t}}}) + 3 W_t(b_t,c_t)$.

\paragraph{(iv)} \textbf{$b_t \in \TOPM$ and $c_t \in \TOPM$}: In that case, there must exist $c \notin \TOPM$ such that $c \in J(t)$. By Property $2$, on event $\cE$ and using $c \in \WORSTM$, we know that $B_{c_t, b_t}(t) = \max_{j \in J(t)} \max_{i \notin J(t)} B_{i,j}(t) \geq \max_{i \notin J(t)} B_{i,c}(t) \geq B_{c_t,c}(t) \geq \mu_{c_t} - \mu_{c} \geq \mu_{c_t} - \mu_{m+1} = {\GAP{{c_t}}}$. Hence, using furthermore Property $1$ yields ${\GAP{{c_t}}} \leq B_{c_t, b_t}(t) \leq W_t(b_t,c_t)$. It follows that, using event $\cE$:
    \begin{eqnarray*}
         B_{c_t, b_t}(t) & \leq & \mu_{c_t} - \mu_{b_t} + 2W_t(b_t,c_t) = \mu_{c_t} - \mu_{m+1} + \mu_{m+1} - \mu_{b_t} + 2W_t(b_t,c_t) \\
         & = & {\GAP{{c_t}}} - {\GAP{{b_t}}} + 2W_t(b_t,c_t) \leq  - {\GAP{{b_t}}} + 3W_t(b_t,c_t)\\
    \end{eqnarray*}
And using again Property $1$, one has:
    \begin{eqnarray*}
       B_{c_t, b_t}(t) & \leq & W_t(b_t,c_t) = - W_t(b_t,c_t) + 2W_t(b_t,c_t) \leq -{\GAP{{c_t}}} + 2W_t(b_t,c_t)\\
        & \leq & -{\GAP{{c_t}}} + 3W_t(b_t,c_t)
    \end{eqnarray*}
Hence $B_{c_t, b_t}(t) \leq -({\GAP{{b_t}}} \lor {\GAP{{c_t}}}) + 3 W_t(b_t,c_t)$, which is what we wanted to show.
\end{proof}

\begin{lemma}\textnormal{\textbf{Upper bound in $m$-LinGapE with either or both $b_t$ and $c_t$ pulled at time $t$ ($m$-LinGapE(1))}}
Maximum number of samplings on event $\cE$ is upper-bound by $\inf_{u \in \bR^{*+}} \{u > 1+\HA{\text{$m$-LinGapE(1)}}C_{\delta,u}^2\}$, where $\HA{\text{$m$-LinGapE(1)}} \triangleq 4\sigma^2\sum_{a \in \ARMS} \max \left( \varepsilon, \frac{\varepsilon+\GAP{a}}{3} \right)^{-2}$.
\end{lemma}

\begin{proof}
Combining Lemma~\ref{lemma:m-LinGapE_bound} with stopping rule $\tauLUCB$, at time $t < \tauLUCB$:

\begin{equation*}
	\begin{split}
	\varepsilon \leq B_{c_{t}, b_{t}}(t) & \leq \min(-(\GAP{b_{t}} \lor \GAP{c_{t}})+3W_t(b_{t},c_{t}), W_t(b_{t},c_{t}))\\
	\Leftrightarrow \max \left( \varepsilon, \frac{\varepsilon+\GAP{b_{t}}}{3}, \frac{\varepsilon+\GAP{c_{t}}}{3} \right)  & \leq W_t(b_{t},c_{t}) \leq W_t(b_{t})+W_t(c_{t}) \leq 2W_t(a_{t}) = 2C_{\delta,t}||x_{a_t}||_{\HATSIGMA{t}}\\
	& \mbox{(where $a_{t} = \max_{a \in \{b_t,c_t\}} W_t(a)$)}\\
	& = 2\sigma C_{\delta,t}||x_{a_t}||_{\INVHATB{t}} \leq 2\sigma C_{\delta,t}\frac{||x_{a_t}||}{\sqrt{\NA{a_t}{t}||x_{a_t}||}}\\
	& = 2\sigma C_{\delta,t}\frac{1}{\sqrt{\NA{a_t}{t}}} \mbox{ (using Lemma~\ref{lemma:paired_versus_individual} and $\lambda > 0, \NA{a_t}{t} > 0$, since $a_t$ is pulled at $t$)}\\
		\Leftrightarrow \NA{a_t}{t} &  \leq \frac{4\sigma^2 C_{\delta,t}^2}{\max \left( \varepsilon, \frac{\varepsilon+\GAP{b_{t}}}{3}, \frac{\varepsilon+\GAP{c_{t}}}{3} \right)^2} \leq \min_{a \in \{b_t,c_t\}} \frac{4\sigma^2 C_{\delta,t}^2}{\max \left( \varepsilon, \frac{\varepsilon+\GAP{a}}{3}\right)^2} \leq \frac{4\sigma^2 C_{\delta,t}^2}{\max \left( \varepsilon, \frac{\varepsilon+\GAP{a_{t}}}{3}\right)^2}\\
	\Leftrightarrow \NA{a_t}{t} &  \leq \frac{4\sigma^2 C_{\delta,t}^2}{\max \left( \varepsilon, \frac{\varepsilon+\GAP{a_{t}}}{3}\right)^2} = T^{*}(a_t, \delta, t)\\
	\end{split}
\end{equation*}

Using Lemma~\ref{lemma:get_upper_bound_sample}, if $T(\mu, \delta)$ is the number of samplings of $m$-LinGapE on bandit instance $\mu$ for $\delta$-fixed confidence Top-$m$ identification:

\begin{equation*}
	\begin{split}
	T(\mu, \delta) & \leq \inf_{u \in \bR^{*+}} \left\{ u > 1+C_{\delta,u}^2\sum_{a \in \ARMS}\frac{4\sigma^2}{\max \left( \varepsilon, \frac{\varepsilon+\GAP{a}}{3}\right)^2}\right\} \leq \inf_{u \in \bR^{*+}} \{ u > 1+C_{\delta,u}^2\HA{\text{$m$-LinGapE(1)}}\}\\
	\end{split}
\end{equation*}
\end{proof}

\section{TECHNICAL LEMMAS}\label{app:technical}

\begin{lemma}\label{lemma:upper_bound_gap}
Let us fix $K > m > 0$, $t > 0$ and $i \in \ARMS$. Let us consider $\mu$ such that $\mu_1 \geq \mu_2 \geq \dots \geq \mu_K$, and a series of distinct values $(B_{j,i}(t))_{j \in \ARMS}$ such that $B_{j,i}(t) \geq \mu_j-\mu_i$ for any $j \in \ARMS$. Then $\maxm{j \in \ARMS} B_{j,i}(t) \geq \mu_m-\mu_i$.
\end{lemma}

\begin{proof}
Assume by appealing to the extremes that $\maxm{j \in \ARMS} B_{j,i}(t) < \mu_{m}-\mu_i$. Then, using our assumption on $(B_{j,i}(t))_{j \in \ARMS}$ and $(\mu_j)_{j \in [K]}$, for any $j \in [m]$, $B_{j,i}(t) \geq \mu_j-\mu_i \geq \mu_m-\mu_i > \maxm{j \in \ARMS} B_{j,i}(t)$, which means at least $m$ distinct values of $(B_{j,i}(t))_{j \leq K}$ are strictly greater than $\maxm{j \in \ARMS} B_{j,i}(t)$, which yields a contradiction. Thus $\maxm{j \in \ARMS} B_{j,i}(t) \geq \mu_m-\mu_i$. Note that we can assume the condition on $(B_{i,j}(t))_{i,j \in \ARMS}$ being distinct is satisfied except for some degenerate cases where two arm features are equal and the observations made from both arms are exactly the same. 
\end{proof}

\begin{lemma}For all $t > 0$, for any subset $J \subseteq [K]$ of size $m$, for all $j\in J$, $\maxm{i\neq j} B_{i,j}(t) \leq \max_{i \notin J} B_{i,j}(t)$. \textnormal{(Lemma~\ref{lem:counting} in the paper)}
\end{lemma}

\begin{proof}
	Indeed, $\maxm{i \neq j} B_{i,j}(t)  = \min_{S \subseteq [K], |S| = m-1} \max_{i \not\in (S\CUP\{j\})} B_{i,j}(t)$ (set $S$ matching the outer bound is $\underset{i \neq j}{\overset{[m-1]}{\argmax}}~B_{i,j}(t)$, meaning that we consider then the maximum value over the set of $(B_{i,j}(t))_{i \in \ARMS}$ from which the $m-1$ largest values and $B_{j,j}(t)$ are removed). Then consider $S = J\setminus\{j\}$, which is included in $[K]$ and is of size $m-1$ ($j \in J$). Then $\maxm{i \neq j} B_{i,j}(t) \leq \max_{i \notin S\cup\{j\}} B_{i,j}(t) = \max_{i \notin J} B_{i,j}(t)$.
\end{proof}

\begin{lemma}\label{lemma:majoration_gap}
	For any $t > 0$, for any $a \in \ARMS$ such that $\NA{a}{t} > 0$, for all $x \in \bR^N$, $||x||^2_{\INVHATB{t}} \leq x^\top (\lambda I_N + \NA{a}{t}x_ax_a^\top )^{-1}x$.
\end{lemma}

\begin{proof}
Let us prove this lemma by induction on $K \geq 2$ (case $K=1$ is trivial). Let $A_t(a) \triangleq \NA{a}{t}x_ax_a^\top $, $A_t \triangleq \lambda I_N + A_t(a)$. For $\ARMS = \{a_1, a_2, \dots, a_{K-1}, a\}$ and $K \geq 2$, let us denote $B^t_{K} \triangleq \sum_{i=1}^{K-1} A_t(a_i)$, such that $\HATB{t} = A_t + B^t_K$. We will prove a stronger claim, which is ``for any $t \in \mathbb{N}^*$, for any $x \in \bR^N$, and $K \geq 2$, $A_t$ and $A_t + B^t_{K}$ are invertible and $||x||^2_{(A_t + B^t_K)^{-1}} < ||x||^2_{A_t^{-1}}$''. Note that, for any $K$ and $t$, since $\lambda > 0$, $A_t$ is then a Gram matrix with linearly independent columns, thus is positive definite, and $B^t_K$ is a Gram matrix, thus a non-negative definite matrix. Then $A_t + B^t_{K}$ and $A_t$ are positive definite and invertible.

\underline{\textbf{If $K = 2$:}} then let us assume that $\ARMS = \{ a, a_1 \}$:

\begin{equation*}
    \begin{split}
        ||x||^2_{(A_t + B^t_2)^{-1}} \triangleq x^\top (A_t + B^t_2)^{-1}x & = x^\top (A_t + \NA{{a_1}}{t}x_{a_1}x_{a_1}^\top )^{-1}x \mbox{ (using Sherman-Morrison formula)}\\
        & = x^\top (A_t^{-1} - \frac{A_t^{-1}\NA{a_1}{t}x_{a_1}x_{a_1}^\top A_t^{-1}}{1+\NA{{a_1}}{t}||x_{a_1}||^2_{A_t^{-1}}})x = ||x||^2_{A_t^{-1}} - \frac{(A_t^{-1}x)^\top B^t_2(A_t^{-1}x)}{1+\NA{a_1}{t}||x_{a_1}||^2_{A_t^{-1}}} \leq ||x||^2_{A_t^{-1}}-0\\
    \end{split}
\end{equation*}

using the fact that $B^t_2$ is nonnegative definite and $A_t$, and then $A_t^{-1}$, are both symmetric.

\underline{\textbf{If $K > 2$:}} using the induction, $A_t + B^t_{K-1}$ is invertible. Similarly to the previous step, using the Sherman-Morrison formula:
    
\begin{equation*}
    \begin{split}
        ||x||^2_{(A_t + B^t_K)^{-1}} &  = x^\top (A_t + B^t_{K-1})^{-1}x - \frac{x^\top (A_t + B^t_{K-1})^{-1}A_t(a_{K-1})(A_t + B^t_{K-1})^{-1}x}{1+\NA{{a_{K-1}}}{t}||x_{a_{K-1}}||^2_{(A_t + B^t_{K-1})^{-1}}}\\
	&  = x^\top (A_t + B^t_{K-1})^{-1}x - \frac{((A_t + B^t_{K-1})^{-1}x)^\top A_t(a_{K-1}) ((A_t + B^t_{K-1})^{-1}x)}{1+\NA{{a_{K-1}}}{t}||x_{a_{K-1}}||^2_{(A_t + B^t_{K-1})^{-1}}}\\
        & \mbox{(same argument as previously, since $A_t(a_{K-1})$ is a Gram matrix)}\\
        & \leq x^\top (A_t + B^t_{K-1})^{-1}x - 0 = ||x||^2_{(A_t + B^t_{K-1})^{-1}}\\
    \end{split}
\end{equation*}
        
Then, using the induction, $||x||^2_{\INVHATB{t}} = ||x||^2_{(A_t + B^t_K)^{-1}} \leq ||x||^2_{(A_t + B^t_{K-1})^{-1}} \leq ||x||^2_{A_t^{-1}} = ||x||^2_{(\lambda I_N + \NA{a}{t}x_ax_a^T)^{-1}}$.

\end{proof}

\begin{lemma}
	$\forall t > 0, \forall a \in \ARMS, \forall y \in \bR^N$, $||y||_{\INVHATB{t}} \leq ||y||/\sqrt{\NA{a}{t}||x_a||^2+\lambda}.$ \textnormal{(Lemma~\ref{lemma:paired_versus_individual} in the paper)}
\end{lemma}

\begin{proof}
Using successively Lemma~\ref{lemma:majoration_gap}, Sherman-Morrison formula and Cauchy-Schwarz inequality:

    \begin{equation*}
        \begin{split}
            ||y||^2_{\INVHATB{t}} & \leq ||y||^2_{(\lambda I_N + N_a(t)x_ax_a^\top )^{-1}}  = \frac{||y||^2}{\lambda} - \frac{\lambda^{-2}\NA{a}{t}(<y,x_{a}>)^2}{1+\lambda^{-1}\NA{{a}}{t}||x_{a}||^2}\\
            & \leq \frac{||y||^2}{\lambda} - \frac{\lambda^{-2}\NA{a}{t}||y||^2||x_a||^2}{1+\lambda^{-1}\NA{{a}}{t}||x_{a}||^2} = \frac{||y||^2}{\lambda+\NA{{a}}{t}||x_{a}||^2}\\
        \end{split}
    \end{equation*}
    
\end{proof}

\begin{lemma}
	Let $T^{*} : \ARMS \times (0,1) \times \mathbb{N}^{*} \rightarrow \bR^{*+}$ be a function that is nondecreasing in $t$, and $\cI_t$ the set of pulled arms at time $t$. Let $\cE$ be an event such that for all $t < \tau_{\delta}, \delta \in (0,1)$, $\exists a_t \in \cI_t, \NA{a_t}{t}\leq T^{*}(a_t, \delta, t) $. Then it holds on the event $\cE$ that $\tau_\delta \leq T(\mu,\delta)$ where \[T(\mu, \delta) \triangleq \inf \left\{u \in \bR^{*+}: u > 1+\sum_{a=1}^{K} T^{*}(a, \delta, u) \right\}.\] \textnormal{(Lemma~\ref{lemma:get_upper_bound_sample} in the paper)}
\end{lemma}

\begin{proof}
Let us denote $T \in \mathbb{N}^{*}$. Let us study $\min(\tau_{\delta}, T)$, because $\min(\tau_{\delta}, T) < T \implies \tau_{\delta} < T$. On event $\cE$:
        
\begin{equation*}
    \begin{split}
        \min(\tau_{\delta}, T) & = 1+ \sum_{t \leq T} \mathds{1}(t < \tau_{\delta}) \leq 1+ \sum_{t \leq T} \mathds{1}(\exists a_t \in \cI_t, \NA{a_t}{t} \leq T^{*}(a_t, \delta, t)) \mbox{ (using definition of $T^{*}$, and $\cE$ holds)}\\
        & = 1+ \sum_{t \leq T} \sum_{m=1}^{t} \mathds{1}(\exists a_t \in \cI_t,\NA{a_t}{t}=m \land m \leq T^{*}(a_t, \delta, t))\mbox{ (using $\forall a \in [K], \NA{a}{t} \in [t] \land \forall a \in \cI_t, \NA{a}{t} > 0$)}\\
        & \leq 1+ \sum_{m=1}^T  \sum_{t=m}^T  \sum_{a \in \ARMS} \mathds{1}(a \in \cI_t)\mathds{1}(\NA{a}{t}=m \land m \leq T^{*}(a, \delta, t)) \mbox{ (using the union bound on pulled arms)}\\
        & = 1+ \sum_{a \in \ARMS} \sum_{m=1}^T  \sum_{t=m}^T  \mathds{1}(a \in \cI_t)\mathds{1}(\NA{a}{t}=m)\mathds{1}(m \leq T^{*}(a, \delta, t))\\
        & \leq 1+ \sum_{a \in \ARMS} \sum_{m=1}^T \left[  \sum_{t=m}^T  \mathds{1}(a \in \cI_t)\mathds{1}(\NA{a}{t}=m) \right] \mathds{1}(m \leq T^{*}(a, \delta, T)) \mbox{ (since $T^{*}$ is nondecreasing in $t$)}\\
        & \leq 1+ \sum_{a \in \ARMS} \sum_{m=1}^T   1 \times \mathds{1}(m \leq T^{*}(a, \delta, T)) \leq 1+ \sum_{a \in \ARMS} T^{*}(a, \delta, T)\\
    \end{split}
\end{equation*}
Choosing any $T$ that satisfies $1+ \sum_{a \in \ARMS} T^{*}(a, \delta, T) < T$ yields $\min(\tau_{\delta}, T) < T$ and therefore $\tau_{\delta} \leq T$. The smallest possible such $T$ is \[T(\mu, \delta) \triangleq \inf \left\{u \in \bR^{*+}: u > 1+ \sum_{a \in \ARMS} T^{*}(a, \delta, u) \right\}.\]
\end{proof}

\end{document}